\documentclass[11pt]{article}
\usepackage{enumerate}
\usepackage[OT1]{fontenc}
\usepackage{amsmath,amssymb}
\usepackage[usenames]{color}
\usepackage{smile}
\usepackage{multirow}
\usepackage[colorlinks,
            linkcolor=red,
            anchorcolor=blue,
            citecolor=blue,
            breaklinks=true
            ]{hyperref}
\usepackage{breakcites}
\def\given{\,|\,}

\def\tr{\mathop{\text{tr}}\kern.2ex}

\long\def\comment#1{}

\def\tr{\mathop{\text{Tr}}}

\def\cS{{\mathcal{S}}}

\newcommand{\bel}{\begin{eqnarray}\label}
\newcommand{\eel}{\end{eqnarray}}
\newcommand{\bes}{\begin{eqnarray*}}
\newcommand{\ees}{\end{eqnarray*}}

\def\real{{\mathbb{R}}}
\def\R{{\real}}

\def\mA{\mathscr{A}}

\usepackage{mathrsfs}

\usepackage{fullpage}

\def\##1\#{\begin{align}#1\end{align}}
\def\$#1\${\begin{align*}#1\end{align*}}
\usepackage{enumitem}

\begin{document}

\title{\huge On the Global Optimality of\\
 Model-Agnostic Meta-Learning}

\author
{
\normalsize Lingxiao Wang\thanks{Northwestern University; \texttt{lingxiaowang2022@u.northwestern.edu}}
\qquad
\normalsize Qi Cai\thanks{Northwestern University; \texttt{qicai2022@u.northwestern.edu}}
\qquad
\normalsize Zhuoran Yang\thanks{Princeton University; \texttt{zy6@princeton.edu}}
\qquad
\normalsize Zhaoran Wang\thanks{Northwestern University; \texttt{zhaoranwang@gmail.com}}
}
\date{\today}

\maketitle


\begin{abstract}
Model-agnostic meta-learning (MAML) formulates meta-learning as a bilevel optimization problem, where the inner level solves each subtask based on a shared prior, while the outer level searches for the optimal shared prior by optimizing its aggregated performance over all the subtasks. 
Despite its empirical success, MAML remains less understood in theory, especially in terms of its global optimality, due to the nonconvexity of the meta-objective (the outer-level objective).
To bridge such a gap between theory and practice, we characterize the optimality gap of the stationary points attained by MAML for both reinforcement learning and supervised learning, where the inner-level and outer-level problems are solved via first-order optimization methods.
In particular, our characterization connects the optimality gap of such stationary points with (i) the functional geometry of inner-level objectives and (ii) the representation power of function approximators, including linear models and neural networks.
To the best of our knowledge, our analysis establishes the global optimality of MAML with  nonconvex meta-objectives for the first time.
\end{abstract}

\section{Introduction}
Meta-learning aims to find a prior that efficiently adapts to a new subtask based on past subtasks. One of the most popular meta-learning methods, namely model-agnostic meta-learning (MAML) \citep{finn2017model},  is based on bilevel optimization, where the inner level solves each subtask based on a shared prior, while the outer level optimizes the aggregated performance of the shared prior over all the subtasks. In particular, MAML associates the solution to each subtask with the shared prior through one step of gradient descent based on the subtask data. Due to its model-agnostic property, MAML is widely adopted in reinforcement learning \citep{finn2017model, finn2017one, xu2018learning, nagabandi2018deep, gupta2018meta, yu2018one, mendonca2019guided} and supervised learning \citep{finn2017model,li2017meta, finn2018probabilistic, rakelly2018few, yoon2018bayesian}.

Despite its popularity in empirical studies, MAML is scarcely explored theoretically. In terms of the global optimality of MAML, \cite{finn2019online} show that the meta-objective is strongly convex assuming that the inner-level objective is strongly convex (in its finite-dimensional parameter). However, such an assumption fails to hold for neural function approximators, which leads to a gap between theory and practice. For nonconvex meta-objectives, \cite{fallah2019convergence} characterize the convergence of MAML to a stationary point under certain regularity conditions. Meanwhile, \cite{rajeswaran2019meta} propose a variant of MAML that utilizes implicit gradients, which is also guaranteed to converge to a stationary point. However, the global optimality of such stationary points remains unclear. On the other hand, \cite{pentina2014pac,amit2017meta} establish PAC-Bayes bounds for the generalization error of two variants of MAML. However, such generalization guarantees only apply to the global optima of the two meta-objectives rather than their stationary points.

In this work, we characterize the global optimality of the $\epsilon$-stationary points attained by MAML for both reinforcement learning (RL) and supervised learning (SL). For meta-RL, we study a variant of MAML, which associates the solution to each subtask with the shared prior, namely $\pi_\theta$, through one step of proximal policy optimization (PPO) \citep{schulman2015trust, schulman2017proximal} in the inner level of optimization. In the outer level of optimization, we maximize the expected total reward associated with the shared prior aggregated over all the subtasks. We prove that the $\epsilon$-stationary point attained by such an algorithm is (approximately) globally optimal given that the function approximator has sufficient representation power. For example, for the linear function approximator $\pi_\theta(s, a) \propto \exp(\phi(s, a)^\top \theta)$, the optimality gap of the $\epsilon$-stationary point is characterized by the representation power of the linear class $\{\phi(\cdot, \cdot)^\top v: v \in \cB\}$, where $\cB$ is the parameter space (which is specified later). The core of our analysis is the functional one-point monotonicity \citep{facchinei2007finite} of the expected total reward $J(\pi)$ with respect to the policy $\pi$ \citep{liu2019neural} for each subtask. Based on a similar notion of functional geometry in the inner level of optimization, we establish similar results on the optimality gap of meta-SL. Moreover, our analysis of both meta-RL and meta-SL allows for neural function approximators. More specifically, we prove that the optimality gap of the attained $\epsilon$-stationary points is characterized by the representation power of the corresponding classes of overparameterized two-layer neural networks.

\vskip4pt
{\noindent\bf Challenge.} We highlight that the bilevel structure of MAML makes it challenging for the analysis of its global optimality. In the simple case where the inner-level objective is strongly convex and smooth, \cite{finn2019online} show that the meta-objective is also strongly convex assuming that the stepsize of inner-level optimization is sufficiently small. 
\begin{itemize}
\item In practice, however, both the inner-level objective and the meta-objective can be nonconvex, which leads to a gap between theory and practice. For example, the inner-level objective of meta-RL is nonconvex even in the (infinite-dimensional) functional space of policies.
\item Even assuming that the inner-level objective is convex in the (infinite-dimensional) functional space, nonlinear function approximators, such as neural networks, can make the inner-level objective nonconvex in the finite-dimensional space of parameters. 
\item Furthermore, even for linear function approximators, the bilevel structure of MAML can make the meta-objective nonconvex in the finite-dimensional space of parameters, especially when the stepsize of inner-level optimization is large. 
\end{itemize}
In this work, we tackle all these challenges by analyzing the global optimality of both meta-RL and meta-SL for both linear and neural function approximators.

\vskip4pt
{\noindent\bf Contribution.} Our contribution is three-fold. First, we propose a meta-RL algorithm and characterize the optimality gap of the $\epsilon$-stationary point attained by such an algorithm for linear function approximators. Second, under an assumption on the functional convexity of the inner-level objective, we characterize the optimality gap of the $\epsilon$-stationary point attained by meta-SL. Finally, we extend our optimality analysis for linear function approximators to handle overparameterized two-layer neural networks. To the best of our knowledge, our analysis establishes the global optimality of MAML with  nonconvex meta-objectives for the first time.
\vskip4pt
{\noindent\bf Related Work.} Meta-learning is studied by various communities \citep{evgeniou2004regularized, thrun2012learning, pentina2014pac,amit2017meta, nichol2018first, nichol2018reptile, khodak2019provable}. See \cite{pan2009survey, weiss2016survey} for the surveys of meta-learning and \cite{taylor2009transfer} for a survey of meta-RL. Our work focuses on the model-agnostic formulation of meta-learning (MAML) proposed by \cite{finn2017model}. In contrast to existing empirical studies, the theoretical analysis of MAML is relatively scarce. \cite{fallah2019convergence} establish the convergence of three variants of MAML for nonconvex meta-objectives. \cite{rajeswaran2019meta} propose a variant of MAML that utilizes implicit gradients of the inner level of optimization and establish the convergence of such an algorithm. This line of work characterizes the convergence of MAML to the stationary points of the corresponding meta-objectives. Our work is complementary to this line of work in the sense that we characterize the global optimality of the stationary points attained by MAML. Meanwhile, \cite{finn2019online} propose an online algorithm for MAML with regret guarantees, which rely on the strong convexity of the meta-objectives. In contrast, our work tackles nonconvex meta-objectives, which allows for neural function approximators, and characterizes the global optimality of MAML. \cite{mendonca2019guided} propose a meta-policy search method and characterize the global optimality for solving the subtasks under the assumption that the meta-objective is (approximately) globally optimal. Our work is complementary to their work in the sense that we characterize the global optimality of MAML in terms of optimizing the meta-objective. See also the concurrent work \citep{meta-learning-convergence-kernel}.

There is a large body of literature that studies the training and generalization of overparameterized neural networks for SL \citep{daniely2017sgd, jacot2018neural, wu2018sgd, allen2018learning, allen2018convergence, du2018gradient1, du2018gradient, zou2018stochastic, chizat2018note, li2018learning, cao2019bounds,  cao2019generalization, arora2019fine, lee2019wide, bai2019beyond}. See \cite{fan2019selective} for a survey. In comparison, we study MAML with overparameterized neural networks for both RL and SL. The bilevel structure of MAML makes our analysis significantly more challenging than that of RL and SL.

\vskip4pt
{\noindent\bf Notation.} We denote by $[n]=\{1, 2, ..., n\}$ the index set. Also, we denote by $x = ([x]_1^\top, \ldots, [x]_m^\top)^\top \in \RR^{md}$ a vector in $\RR^{md}$, where $[x]_k\in \RR^d$ is the $k$-th block of $x$ for $k\in[m]$. For a real-valued function $f$ defined on $\cX$, we denote by $\|f(\cdot)\|_{p, \nu} = \{\int_\cX f^p(x) \ud \nu(x)\}^{1/p}$ the $L_p(\nu)$-norm of $f$, where $\nu$ is a measure on $\cX$. We write $\|f(\cdot)\|_{2, \nu} = \|f(\cdot)\|_\nu$ for notational simplicity and $\|f\|_{p, \nu} = \|f(\cdot)\|_{p, \nu}$ when the variable is clear from the context. For a vector $\phi \in \RR^n$, we denote by $\|\phi\|_2$ the $\ell_2$-norm of $\phi$.

\section{Background}
In this section, we briefly introduce reinforcement learning and meta-learning.
\subsection{Reinforcement Learning}
We define a Markov decision process (MDP) by a tuple $(\cS, \cA, P, r, \gamma, \zeta)$, where $\cS$ and $\cA$ are the state and action spaces, respectively, $P$ is the Markov kernel, $r$ is the reward function, which is possibly stochastic, $\gamma\in(0, 1)$ is the discount factor, and $\zeta$ is the initial state distribution over $\cS$. In the sequel, we assume that $\cA$ is finite. An agent interacts with the environment as follows. At each step $t$, the agent observes the state $s_t$ of the environment, takes the action $a_t$, and receives the reward $r(s_t, a_t)$. The environment then transits into the next state according to the distribution $P(\cdot\given s_t, a_t)$ over $\cS$. We define a policy $\pi$ as a mapping from $\cS$ to distributions over $\cA$. Specifically, $\pi(a\given s)$ gives the probability of taking the action $a$ at the state $s$. Given a policy $\pi$, we define for all $(s, a)\in\cS\times\cA$ the corresponding state- and action-value functions $V^\pi$ and $Q^\pi$ as follows,
\#
\label{eq::def_V} V^\pi(s) &= (1 - \gamma)\cdot \EE\biggl[\sum^\infty_{t = 0} \gamma^t \cdot r(s_t, a_t) \,\biggl|\, s_0 = s\biggr],\\
\label{eq::def_Q} Q^\pi(s, a) &= (1 - \gamma)\cdot \EE\biggl[\sum^\infty_{t = 0}\gamma^t \cdot r(s_t, a_t)\,\biggl|\, s_0 = s, a_0 = a\biggr],
\#
where $s_{t+1}\sim P(\cdot\given s_t, a_t)$ and $a_t\sim \pi(\cdot\given s_t)$ for all $t \geq 0$. Correspondingly, the advantage function $A^\pi$ is defined as follows,
\#\label{eq::def_A}
A^\pi(s, a) = Q^\pi(s, a) - V^\pi(s),\quad \forall (s, a)\in\cS\times\cA.
\#
A policy $\pi$ induces a state visitation measure $\nu_\pi$ on $\cS$, which takes the form of
\#\label{eq::def_visit}
\nu_\pi(s) = (1 - \gamma)\cdot \sum^\infty_{t = 0} \gamma^t \cdot \PP(s_t = s),
\#
where $s_0 \sim \zeta$, $s_{t+1}\sim P(\cdot\given s_t, a_t)$, and $a_t \sim \pi(\cdot\given s_t)$ for all $t \geq 0$. Correspondingly, we define the state-action visitation measure by $\sigma_\pi(s, a) = \pi(a\given s)\cdot\nu_\pi(s)$ for all $(s, a)\in\cS\times \cA$, which is a probability distribution over $\cS\times\cA$. The goal of reinforcement learning is to find the optimal policy $\pi^*$ that maximizes the expected total reward $J(\pi)$, which is defined as
\#\label{eq::def_J}
J(\pi) = \EE_{s\sim \zeta}\bigl[V^\pi(s)\bigr] = \EE_{(s, a)\sim \sigma_\pi}\bigl[r(s, a)\bigr].
\#
When $\cS$ is continuous, maximizing $J(\pi)$ over all possible $\pi$ is computationally intractable. A common alternative is to parameterize the policy by $\pi_\theta$ with the parameter $\theta\in\Theta$, where $\Theta$ is the parameter space, and maximize $J(\pi_\theta)$ over $\theta\in\Theta$.
\subsection{Meta-Learning}

In meta-learning, the meta-learner is given a sample of learning subtasks $\{\cT_i\}_{i\in[n]}$ drawn independently from the task distribution $\iota$ and a set of parameterized algorithms $\cA = \{\mA_\theta : \theta \in \Theta\}$, where $\Theta$ is the parameter space. Specifically, given $\theta$, the algorithm $\mA_\theta \in\cA$ maps from a learning subtask $\cT$ to its desired outcome. For example, an algorithm that solves reinforcement learning subtasks maps from an MDP $\cT = (\cS, \cA, P, r, \gamma, \zeta)$ to a policy $\pi$, aiming at maximizing the expected total reward $J(\pi)$ defined in \eqref{eq::def_J}. As an example, given a hypothesis class $\cH$, a distribution $\cD$ over $\cZ$, which is the space of data points, and a loss function $\ell:\cH\times\cZ \mapsto \RR$, a supervised learning subtask aims at minimizing the risk $\EE_{z\sim \cD}[\ell(h, z)]$ over $h\in\cH$. We denote the supervised learning subtask $\cT$ by the tuple $(\cD,  \ell, \cH)$. Similarly, an algorithm that solves supervised learning subtasks maps from $\cT = (\cD, \ell, \cH)$ to a hypothesis $h \in \cH$, aiming at minimizing the risk $R(h) = \EE_{z\sim \cD}[\ell(h, z)]$ over $h \in \cH$. In what follows, we denote by $H_{\cT}$ the objective of a learning subtask $\cT$. If $\cT$ is a reinforcement learning subtask, we have $H_{\cT} = J$, and if $\cT$ is a supervised learning subtask, we have $H_{\cT} = R$.

The goal of the meta-learner is to find $\theta^*\in\Theta$ that optimizes the population version of the meta-objective $\overline L(\theta)$, which is defined as 
\#\label{eq::def_meta_L}
\overline L(\theta) = \EE_{\cT\sim \iota} \Bigl[H_{\cT}\bigl(\mA_{\theta}(\cT)\bigr)\Bigr].
\#
To approximately optimize $\overline L$ defined in \eqref{eq::def_meta_L} based on the sample $\{\cT_i\}_{i\in[n]}$ of subtasks, the meta-learner optimizes the following empirical version of the meta-objective,
\#\label{eq::def_meta_eL}
L(\theta) = \frac{1}{n}\cdot \sum^n_{i = 1}H_{\cT_i}\bigl(\mA_{\theta}(\cT_i)\bigr).
\#
The algorithm $\mA_{\theta^*}$ corresponding to the global optimum $\theta^*$ of \eqref{eq::def_meta_eL} incorporates the past experience through the observed learning subtasks $\{\cT_i\}_{i\in[n]}$, and therefore, facilitates the learning of a new subtask \citep{pentina2014pac, finn2017model, amit2017meta, yoon2018bayesian}.~As an example, in model-agnostic meta-learning (MAML) \citep{finn2017model} for supervised learning, the hypothesis class $\cH$ is parameterized by $h_\theta$ with $\theta \in \Theta$, and the algorithm $\mA_\theta$ performs one step of gradient descent with $\theta\in\Theta$ as the starting point. In this setting, MAML aims to find the globally optimal starting point $\theta^*$ by minimizing the following meta-objective by gradient descent,
\$
L(\theta) = \frac{1}{n}\cdot \sum^n_{i = 1}R_i \bigl( h_{\theta - \eta\cdot\nabla_\theta R_i(h_\theta)}\bigr),
\$
where $\eta$ is the learning rate of $\mA_\theta$ and $R_i(h) = \EE_{z\sim\cD_i}[\ell(h, z)]$ is the risk of the supervised learning subtask $\cT_i = (\cD_i, \ell, \cH)$.~Similarly, in MAML for reinforcement learning, the algorithm $\mA_\theta$ performs, e.g., one step of policy gradient with $\theta$ as the starting point. We call $\pi_\theta$ the main effect in the sequel. MAML aims to find the globally optimal main effect $\pi_{\theta^*}$ by maximizing the following meta-objective by gradient ascent,
\$
L(\theta) = \frac{1}{n}\cdot \sum^n_{i = 1}J_i\bigl(\pi_{\theta + \eta\cdot\nabla_\theta J_i(\pi_\theta)}\bigr),
\$
where $\eta$ is the learning rate of $\mA_\theta$ and $J_i$ is the expected total reward of the reinforcement learning subtask $\cT_i = (\cS, \cA, P_i, r_i, \gamma_i, \zeta_i)$. 

\section{Meta-Reinforcement Learning}
\label{sec::MRL}
In this section, we present the analysis of meta-reinforcement learning (meta-RL). We first define the detailed problem setup of meta-RL and propose a meta-RL algorithm. We then characterize the global optimality of the stationary point attained by such an algorithm. 
\subsection{Problem Setup and Algorithm}
In meta-RL, the meta-learner observes a sample of MDPs $\{(\cS, \cA, P_i, r_i, \gamma_i, \zeta_i)\}_{i\in[n]}$ drawn independently from a task distribution $\iota$. We set the algorithm $\mA_\theta$ in \eqref{eq::def_meta_eL}, which optimizes the policy, to be one step of (a variant of) proximal policy optimization (PPO) \citep{schulman2015trust, schulman2017proximal} starting from the main effect $\pi_\theta$. More specifically, $\mA_\theta$ solves the following maximization problem,
\#\label{eq::def_PPO}
&\mA_\theta(\cS, \cA, P_i, r_i, \gamma_i, \zeta_i) \notag\\
&\quad =\argmax_{\pi} \EE_{s \sim \nu_{i, \pi_\theta}}\Bigl[\langle Q^{\pi_\theta}_i(s, \cdot), \pi(\cdot\given s) \rangle- 1/\eta \cdot D_{\text{\rm KL}}\bigl(\pi(\cdot\given s )\big\| \pi_\theta(\cdot \given s)\bigr)\Bigr].
\#
Here $\langle\cdot, \cdot\rangle$ is the inner product over $\RR^{|\cA|}$, $\eta$ is the tuning parameter of $\mA_\theta$, and $Q^{\pi_\theta}_i$, $\nu_{i, \pi_\theta}$ are the action-value function and the state visitation measure, respectively, corresponding to the MDP $(\cS, \cA, P_i, r_i, \gamma_i, \zeta_i)$ and the policy $\pi_\theta$. Note that the objective in \eqref{eq::def_PPO} has $D_{\text{\rm KL}}(\pi(\cdot\given s)\| \pi_\theta(\cdot\given s))$ in place of $D_{\text{\rm KL}}(\pi_\theta(\cdot\given s)\| \pi(\cdot\given s))$ compared with the original version of PPO \citep{schulman2015trust, schulman2017proximal}. As shown by \cite{liu2019neural}, such a variant of PPO enjoys global optimality and convergence.


We parameterize the main effect $\pi_\theta$ as the following energy-based policy \citep{haarnoja2017reinforcement},
\#\label{eq::def_leader}
\pi_\theta(a\given s) = \frac{\exp\bigl(1/\tau\cdot \phi(s, a)^\top \theta \bigr)}{\sum_{a'\in\cA}\exp\bigl(1/\tau\cdot\phi(s, a')^\top \theta \bigr)}, \quad \forall (s, a)\in\cS\times\cA,
\#
where $\phi: \cS\times\cA \mapsto \RR^d$ is the feature mapping, $\theta \in \RR^d$ is the parameter, $\phi(\cdot, \cdot)^\top \theta$ is the energy function, and $\tau$ is the temperature parameter. The maximizer $\pi_{i, \theta}=\mA_\theta(\cS, \cA, P_i, r_i, \gamma_i, \zeta_i) $ defined in \eqref{eq::def_PPO} then takes the following form \citep[Proposition~3.1]{liu2019neural},
\#\label{eq::PPO_optimal}
\pi_{i, \theta}(\cdot \given s) \propto \exp\bigl(1/\tau\cdot \phi(s, \cdot)^\top \theta + \eta \cdot Q^{\pi_\theta}_i(s, \cdot)\bigr), \quad \forall s \in \cS.
\#
The goal of meta-RL is to find the globally optimal main effect $\pi_\theta$ by maximizing the following meta-objective,
\#\label{eq::def_meta_obj_rl}
L(\theta) = \frac{1}{n}\cdot \sum^n_{i = 1} J_i(\pi_{i, \theta}), \quad \text{\rm where}~\pi_{i,\theta} = \mA_\theta(\cS, \cA, P_i, r_i, \gamma_i, \zeta_i).
\#
Here $J_i$ is the expected total reward defined in \eqref{eq::def_J} corresponding to the MDP $(\cS, \cA, P_i, r_i, \gamma_i, \zeta_i)$ for all $i\in[n]$. To maximize $L(\theta)$, we use gradient ascent, which iteratively updates $\theta$ as follows,
\#\label{eq::update}
\theta_{\ell+1} \leftarrow \theta_\ell + \alpha_\ell \cdot \nabla_\theta L(\theta_\ell), \quad\text{\rm for}~\ell = 0, 1, \ldots, T-1,
\#
where $\nabla_\theta L(\theta_\ell)$ is the gradient of the meta-objective at $\theta_\ell$, $\alpha_\ell$ is the learning rate at the $\ell$-th iteration, and $T$ is the number of iterations. It remains to calculate the gradient $\nabla_\theta L(\theta)$. To this end, we first define the state-action visitation measures induced by the main effect $\pi_{\theta}$, and then calculate $\nabla_\theta L(\theta)$ in closed form based on such state-action visitation measures.
\begin{definition}[Visitation Measures of Main Effect]
\label{def::meta_visit}
For all $i\in[n]$, given the MDP $(\cS, \cA, P_i, r_i, \gamma_i, \zeta_i)$ and the main effect $\pi_\theta$, we denote by $\sigma_{i, \pi_\theta}$ the state-action visitation measure induced by the main effect $\pi_\theta$. We further define the state-action visitation measure $\sigma^{(s, a)}_{i, \pi_\theta}$ initialized at $(s, a) \in \cS\times\cA$ as follows,
\#\label{eq::def_visit_m_init}
\sigma^{(s, a)}_{i, \pi_\theta}(s', a') = (1 - \gamma_i)\cdot \sum^\infty_{t = 0} \gamma_i^t \cdot \PP(s_t = s', a_t = a'),\quad \forall (s', a')\in\cS\times\cA,
\#
where $s_0 \sim P_i(\cdot\given s, a)$, $s_{t+1}\sim P_i(\cdot\given s_t, a_t)$, and $a_t \sim \pi_\theta(\cdot \given s_t)$ for all $t \geq 0$.
\end{definition}
In other words, given the transition kernel $P_i$ and the discount factor $\gamma_i$, $\sigma^{(s, a)}_{i, \pi_\theta}$ 
is the state-action visitation measure induced by the main effect $\pi_\theta$ where the initial state distribution is given by $s_0\sim P_i(\cdot\given s, a)$. 
Based on the policy gradient theorem \citep{sutton2018reinforcement}, the following proposition calculates the gradient of the meta-objective $L$ defined in \eqref{eq::def_meta_obj_rl} with respect to the parameter $\theta$ of the main effect $\pi_\theta$.
\begin{proposition}[Gradient of Meta-Objective]
\label{prop::meta_grad}
It holds for all $\theta\in\RR^d$ that
\#\label{eq::meta_grad_thm}
&\nabla_\theta L(\theta) = \frac{1}{n}\cdot \sum^n_{i = 1} \EE_{(s, a)\sim \sigma_{\pi_{i, \theta}}}\bigl[h_{i, \theta}(s, a)\cdot A_i^{\pi_{i, \theta}}(s, a)\bigr],
\#%
where the auxiliary function $h_{i, \theta}$ takes the form of
\#\label{eq::def_h}
h_{i, \theta}(s, a) 
= 1/\tau\cdot\phi(s, a) + \eta\cdot\gamma_i/\tau\cdot \EE_{(s', a')\sim \sigma^{(s, a)}_{i, \pi_\theta}}\bigl[\phi(s', a')\cdot A^{\pi_{\theta}}_i(s', a')\bigr].
\#
Here $A^{\pi_{i, \theta}}_i$ and $A_i^{\pi_\theta}$ are the advantage functions of the policy $\pi_{i, \theta}$ and the main effect $\pi_\theta$, respectively, both corresponding to the MDP $(\cS, \cA, P_i, r_i, \gamma_i, \zeta_i)$. Also, $\sigma^{(s, a)}_{i, \pi_\theta}$ is the state-action visitation measure induced by the main effect $\pi_\theta$ defined in Definition \ref{def::meta_visit}, and $\sigma_{\pi_{i, \theta}}$ is the state-action visitation measure induced by the policy $\pi_{i, \theta}$, both corresponding to the MDP $(\cS, \cA, P_i, r_i, \gamma_i, \zeta_i)$.
\end{proposition}
\begin{proof}
See \S\ref{pf::meta_grad} for a detailed proof.
\end{proof}
In the sequel, we assume without loss of generality that the action-value function $Q^\pi$ is available once we obtain the policy $\pi$, and the expectations over state-action visitation measures in \eqref{eq::meta_grad_thm} and \eqref{eq::def_h} of Theorem \ref{prop::meta_grad} are available once we obtain the policies $\{\pi_{i, \theta}\}_{i\in[n]}$ and the main effect $\pi_\theta$. We summarize meta-RL in Algorithm \ref{alg::meta_rl}.~In practice, we can estimate the action-value functions by temporal difference learning \citep{sutton1988learning} and the expectations over the visitation measures by Monte Carlo sampling \citep{konda2002actor}.
\begin{algorithm}[htpb]
\caption{Meta-RL}
\begin{algorithmic}[1]\label{alg::meta_rl}
\REQUIRE MDPs $\{v\}_{i\in[n]}$ sampled from the task distribution $\iota$, feature mapping $\phi$, number of iterations $T$, learning rate $\{\alpha_\ell\}_{\ell\in[T]}$,  temperature parameter $\tau$, tuning parameter $\eta$, initial parameter $\theta_0$. 
\FOR{$\ell = 0, \ldots, T-1$}
\FOR{$i \in [n]$}
\STATE Obtain the action-value function $Q^{\pi_{\theta_\ell}}_i$ and the advantage function $A^{\pi_{\theta_\ell}}_i$ corresponding to the MDP $(\cS, \cA, P_i, r_i, \gamma_i, \zeta_i)$ and the main effect $\pi_{\theta_\ell}$. 
\STATE Update the policy $\pi_{i, \theta_\ell}(\cdot \given s) \propto \exp\bigl(1/\tau\cdot \phi(s, \cdot)^\top \theta_\ell + \eta \cdot Q^{\pi_{\theta_\ell}}_i(s, \cdot)\bigr)$.
\STATE Obtain the advantage function $A^{\pi_{i, \theta_\ell}}_i$ corresponding to the MDP  $(\cS, \cA, P_i, r_i, \gamma_i, \zeta_i)$.
\STATE {Compute the auxiliary function \$
h_{i, \theta_\ell}(s, a) \leftarrow 1/\tau\cdot\phi(s, a) + \gamma_i\cdot\eta/\tau\cdot \EE_{(s', a')\sim \sigma^{(s, a)}_{i, \pi_{\theta_\ell}}}\bigl[\phi(s', a')\cdot A^{\pi_{\theta_\ell}}_i(s', a')\bigr].\$}
\ENDFOR
\STATE Compute the gradient of the meta-objective \$
\nabla_\theta L(\theta_\ell) \leftarrow \frac{1}{n}\cdot\sum^n_{i = 1} \EE_{(s, a)\sim \sigma_{\pi_{i, \theta_\ell}}}\bigl[h_{i, \theta_\ell}(s, a)\cdot A_i^{\pi_{i, \theta_\ell}}(s, a)\bigr].
\$
\STATE Update the parameter of the main effect $\theta_{\ell+1} \leftarrow \theta_\ell + \alpha_\ell\cdot\nabla_\theta L(\theta_\ell)$.
\STATE Update the main effect $\pi_{\theta_{\ell+1}}(\cdot\given s) \propto \exp\bigl(1/\tau\cdot \phi(s, \cdot)^\top \theta_{\ell+1}\bigr)$.
\ENDFOR
\STATE {\bf Output:} $\theta_T$ and $\pi_{\theta_T}$.
\end{algorithmic}
\end{algorithm}

\subsection{Theoretical Results}
In this section, we analyze the global optimality of the $\epsilon$-stationary point attained by meta-RL (Algorithm \ref{alg::meta_rl}). In the sequel, we assume that the reward functions $\{r_i\}_{i\in[n]}$ are upper bounded by an absolute constant $Q_{\max} > 0$ in absolute value. It then follows from \eqref{eq::def_V} and \eqref{eq::def_Q} that $|V^\pi_i(s, a)|$ and $|Q^\pi_i(s, a)|$ are upper bounded by $Q_{\max}$ for all $i\in[n]$ and $(s, a)\in\cS\times\cA$. Here we define $Q^\pi_i$ and $V^\pi_i$ as the state- and action-value functions of the policy $\pi$, respectively, corresponding to the MDP $(\cS, \cA, P_i, r_i, \gamma_i, \zeta_i)$.

To analyze the global optimality of meta-RL, we define the following meta-visitation measures induced by the main effect $\pi_\theta$.
\begin{definition}[Meta-Visitation Measures]
\label{def::meta_meta_visit}
We define the joint meta-visitation measure $\rho_{i, \pi_\theta}$ induced by the main effect $\pi_\theta$ and the policy $\pi_{i, \theta}$ as follows,
\#\label{eq::def_meta_joint}
\rho_{i, \pi_\theta}(s', a', s, a) = \sigma^{(s, a)}_{i, \pi_\theta}(s', a')\cdot \sigma_{\pi_{i, \theta}}(s, a), \quad\forall (s', a', s, a)\in\cS\times\cA\times\cS\times\cA.
\#
We further define the meta-visitation measure $\varsigma_{i, \pi_\theta}$ as the marginal distribution of the joint meta-visitation measure $\rho_{i, \pi_\theta}$ of $(s', a')$, that is,
\#\label{eq::def_meta_visit}
\varsigma_{i, \pi_\theta}(s', a') = \EE_{(s, a)\sim \sigma_{\pi_{i, \theta}}}\bigl[\sigma^{(s, a)}_{i, \pi_\theta}(s', a')\bigr],\quad \forall (s', a')\in \cS\times\cA.
\#
In addition, we define the mixed meta-visitation measure $\varrho_{\pi_\theta}$ over all the subtasks as follows,
\#\label{eq::def_mixed_meta}
\varrho_{\pi_\theta}(s', a') = \frac{1}{n} \cdot\sum^n_{i = 1}\varsigma_{i, \pi_\theta}(s', a'),\quad \forall (s', a')\in \cS\times\cA.
\#
\end{definition}
In other words, the meta-visitation measure $\varsigma_{i, \pi_\theta}$ is the state-action visitation measure induced by $\pi_\theta$ given the transition kernel $P_i$, the discount factor $\gamma_i$, and the initial state distribution $s_0\sim \EE_{(s, a)\sim \sigma_{\pi_{i, \theta}}}[P_i(\cdot\given s, a)]$.

In what follows, we impose an assumption on the meta-visitation measures defined in Definition \ref{def::meta_meta_visit}.
\begin{assumption}[Regularity Condition on Meta-Visitation Measures]
\label{asu::concen_coeff}
We assume for all $\theta \in \RR^d$ and $ i \in [n]$ that
\#
\label{eq::def_C_0_1}\EE_{(s', a')\sim \varrho_{\pi_\theta} }\Bigl[\bigl(\ud \sigma_{\pi_{i, \theta}} / \ud \varrho_{\pi_\theta}(s', a')\bigr)^2 \Bigr] \leq C_0^2,\\
\label{eq::def_C_0_2}\EE_{(s', a')\sim \varrho_{\pi_\theta} }\Bigl[\bigl(\ud \varsigma_{i, \pi_{\theta}} / \ud \varrho_{\pi_\theta}(s', a')\bigr)^2 \Bigr] \leq C_0^2, \#
where $C_0>0$ is an absolute constant . Here $\varsigma_{i, \pi_{\theta}}$ and $\varrho_{\pi_\theta}$ are the meta-visitation measure and the mixed meta-visitation measure induced by the main effect $\pi_\theta$, which are defined in \eqref{eq::def_meta_visit}  and \eqref{eq::def_mixed_meta} of Definition \ref{def::meta_meta_visit}, respectively. Meanwhile, $\sigma_{\pi_{i, \theta}}$ is the state-action visitation measure induced by the policy $\pi_{i, \theta}$, which is defined in \eqref{eq::def_visit}.
Here $\ud \sigma_{\pi_{i, \theta}} / \ud \varrho_{\pi_\theta}$ and $\ud  \varsigma_{i, \pi_{\theta}} / \ud \varrho_{\pi_\theta}$ are the Radon-Nikodym derivatives.
\end{assumption}
According to \eqref{eq::def_mixed_meta} of Definition \ref{def::meta_meta_visit}, the upper bound in \eqref{eq::def_C_0_1} of Assumption \ref{asu::concen_coeff} holds if the $L_2(\varrho_{\pi_\theta})$-norms of $\ud \sigma_{\pi_{i, \theta}} / \ud \varsigma_{j, \pi_\theta}$ is upper bounded by $C_0$ for all $i, j \in [n]$. 
For $i = j$, note that $\pi_{i, \theta}$ is obtained by one step of PPO with $\pi_\theta$ as the starting point. Thus, for a sufficiently small tuning parameter $\eta$ in \eqref{eq::PPO_optimal}, $\pi_{i, \theta}$ is close to $\pi_\theta$. Hence, the assumption that $\ud \sigma_{\pi_{i, \theta}} / \ud \varsigma_{j, \pi_\theta}$ has an upper bounded $L_2(\varrho_{\pi_\theta})$-norm for all $i=j$ is a mild regularity condition. For $i \neq j$, to ensure the upper bound of the $L_2(\varrho_{\pi_\theta})$-norms of $\ud \sigma_{\pi_{i, \theta}} / \ud \varsigma_{j, \pi_\theta}$ in \eqref{eq::def_C_0_1}, Assumption \ref{asu::concen_coeff} requires the task distribution $\iota$ to generate similar MDPs so that the meta-visitation measures $\{\varsigma_{i, \pi_\theta}\}_{i\in[n]}$ are similar across all the subtasks indexed by $i\in[n]$. Similarly, to ensure the upper bound in \eqref{eq::def_C_0_2}, Assumption \ref{asu::concen_coeff} also requires that the meta-visitation measures $\{\varsigma_{i, \pi_\theta}\}_{i\in[n]}$ are similar across all the subtasks indexed by $i\in[n]$.

The following theorem characterizes the optimality gap of the $\epsilon$-stationary point attained by meta-RL (Algorithm \ref{alg::meta_rl}). Let $\theta^*$ be a global maximizer of the meta-objective $L(\theta)$ defined in \eqref{eq::def_meta_obj_rl}. For all $(s', a')\in\cS\times\cA$ and $\omega \in \RR^d$, we define
\#\label{eq::def_f}
f_\omega(s', a') = \biggl(\sum^n_{i = 1}\frac{A^{\pi_{i, \omega}}_i(s', a')}{1 - \gamma_i}\cdot \frac{\ud \sigma_{\pi_{i, \theta^*}}}{ \ud \varrho_{\pi_\omega}}(s', a')\biggr)\bigg/\biggl(\sum^n_{i = 1} g_{i, \omega}(s', a') \cdot \frac{\ud \varsigma_{i, \pi_{\omega}}}{ \ud \varrho_{\pi_\omega}}(s', a')\biggr),
\#
where we defined $g_{i, \omega}$ as follows, 
\$
g_{i, \omega}(s', a') = 1/\tau \cdot A^{\pi_{i, \omega}}_i(s', a') \cdot (\ud \sigma_{\pi_{i, \omega}}/\ud \varsigma_{i, \pi_\omega})(s', a') + \gamma_i\cdot\eta/\tau\cdot G_{i, \pi_\omega}(s', a')\cdot A^{\pi_\omega}_i(s', a').
\$
Here $\tau$ is the temperature parameter in \eqref{eq::def_leader}, $\eta$ is the tuning parameter defined in \eqref{eq::def_PPO}, $A^{\pi_{i, \omega}}_i$ and $A^{\pi_\omega}_i$ are the advantage functions of the policy $\pi_{i, \omega}$ and the main effect $\pi_\omega$, respectively, corresponding to the MDP $(\cS, \cA, P_i, r_i, \gamma_i, \zeta_i)$, and $G_{i, \pi_\omega}$ is defined as follows,
\#\label{eq::meta_grad_def_A}
G_{i, \pi_\omega}(s', a') = \EE_{(s', a', s, a)\sim \rho_{i, \pi_\omega}}\bigl[A^{\pi_{i, \omega}}_i(s, a) \,\big|\, s', a'\bigr],
\#
where $\rho_{i, \pi_\omega}$ is the joint meta-visitation measure defined in \eqref{eq::def_meta_joint} of Definition \ref{def::meta_meta_visit}.

\begin{theorem}[Optimality Gap of $\epsilon$-Stationary Point]
\label{thm::eps_opt}

Under Assumption \ref{asu::concen_coeff}, for all $R > 0$, $\omega \in \RR^d$, and $\epsilon > 0$ such that
\$
\nabla_\omega L(\omega)^\top v \leq \epsilon, \quad\forall v \in \cB = \{\theta \in \RR^d: \|\theta\|_2 \leq 1\},
\$
we have
\#\label{eq::no_spurious}
L(\theta^*) - L(\omega) &\leq R\cdot\epsilon + 2C_0\cdot Q_{\max}/\tau\cdot(1 + 2Q_{\max}\cdot\overline\gamma\cdot\eta)\cdot \inf_{v\in\cB_R} \|f_{\omega}(\cdot, \cdot) - \phi(\cdot, \cdot)^\top v\|_{\varrho_{\pi_\omega}},
\#
where $\cB_R = \{\theta\in\RR^d: \|\theta\|_2\leq R\}$, $\overline \gamma = (\sum^n_{i = 1}\gamma_i)/n$, $C_0$ is defined in Assumption \ref{asu::concen_coeff}, $\tau$ is the temperature parameter in \eqref{eq::def_leader}, $\eta$ is the tuning parameter defined in \eqref{eq::def_PPO}, and $Q_{\max}$ is the upper bound of the reward functions $\{r_i\}_{i\in[n]}$ in absolute value.
\end{theorem}
\begin{proof}
See \S\ref{pf::thm_eps_opt} for a detailed proof.
\end{proof}
By Theorem \ref{thm::eps_opt}, the global optimality of the $\epsilon$-stationary point $\omega$ hinges on the representation power of the linear class $\{\phi(\cdot)^\top\theta: \theta \in\cB_R\}$. More specifically, if the function $f_\omega$ defined in \eqref{eq::def_f} is well approximated by $\phi(\cdot)^\top \theta$ for a parameter $\theta\in\cB_R$, then $\omega$ is approximately globally optimal.

\section{Meta-Supervised Learning}
\label{sec::MSL}
In this section, we present the analysis of meta-supervised learning (meta-SL). We first define the detailed problem setup of meta-SL and present a meta-SL algorithm. We then characterize the global optimality of the stationary point attained by such an algorithm. 
\subsection{Problem Setup and Algorithm}
In meta-SL, the meta-learner observes a sample of supervised learning subtasks $\{(\cD_i, \ell, \cH)\}_{i \in [n]}$ drawn independently from a task distribution $\iota$. Specifically, each subtask $(\cD_i, \ell, \cH)$ consists of a distribution $\cD_i$ over $\cX\times\cY$, where $\cY\subseteq \RR$, a loss function $\ell:\cH\times\cX\times\cY\mapsto\RR$, and a hypothesis class $\cH$. Each hypothesis $h \in \cH$ is a mapping from $\cX$ to $\cY$. The goal of the supervised learning subtask $(\cD_i, \ell, \cH)$ is to obtain the following hypothesis,
\#\label{eq::sl_task}
h^*_i = \argmin_{h\in\cH} R_i(h) = \argmin_{h\in\cH} \EE_{z\sim \cD_i}\bigl[ \ell(h, z)\bigr],
\# 
where $R_i(h) = \EE_{z\sim \cD_i}[ \ell(h, z)]$ is the risk of $h\in\cH$. To approximately attain the minimizer defined in \eqref{eq::sl_task}, we parameterize the hypothesis class $\cH$ by $\cH_\theta$ with a feature mapping $\phi:\cX\mapsto \RR^d$ as follows, \#\label{eq::def_H_theta}
\cH_{\theta} = \bigl\{ h_\theta(\cdot) = \phi(\cdot)^\top\theta : \theta \in \RR^d\bigr\},
\#
and minimize $R_i(h_\theta)$ over $\theta\in\RR^d$. We set the algorithm $\mA_\theta$ in \eqref{eq::def_meta_eL}, which solves $(\cD_i, \ell, \cH)$, to be one step of gradient descent with the starting point  $\theta$, that is,
\#\label{eq::def_sl_alg}
\mA_\theta(\cD_i, \ell, \cH) = h_{\theta - \eta \cdot \nabla_\theta R_i(h_\theta)}.
\#
Here $\eta$ is the learning rate of $\mA_\theta$. The goal of meta-SL is to minimize the following meta-objective,
\#\label{eq::def_sl_meta_obj}
L(\theta) = \frac{1}{n}\cdot \sum^n_{i = 1} R_i( h_{\theta_i} ), \quad \text{\rm where}~h_{\theta_i} = \mA_\theta(\cD_i, \cR, \cH).
\# 
To minimize $L(\theta)$ defined in \eqref{eq::def_sl_meta_obj}, we adopt gradient descent, which iteratively updates $\theta_\ell$ as follows,
\#\label{eq::def_maml_sl_alg}
\theta_{\ell+1} \leftarrow \theta_{\ell} - \alpha_\ell\cdot \nabla_\theta L(\theta_\ell), \quad\text{\rm for}~\ell = 0, 1, \ldots, T-1.
\#
Here $\nabla_\theta L(\theta_\ell)$ is the gradient of the meta-objective at $\theta_\ell$, $\alpha_\ell$ is the learning rate at the $\ell$-th iteration, and $T$ is the number of iterations.~\cite{fallah2019convergence} show that the update defined in \eqref{eq::def_maml_sl_alg} converges to an $\epsilon$-stationary point of the meta-objective $L$ under a smoothness assumption on $L$. In what follows, we characterize the optimality gap of such an $\epsilon$-stationary point.

We first introduce the Fr\'echet differentiability of the risk $R_i$ in \eqref{eq::sl_task}.
\begin{definition}[Fr\'echet Differentiability]
\label{def::fre_diff}
Let $\cH$ be a Banach space with the norm $\|\cdot\|_\cH$. A functional $R: \cH \mapsto \RR$ is Fr\'echet differentiable at $h \in \cH$ if it holds for a bounded linear operator $A: \cH \mapsto \RR$ that
\#\label{eq::Fre_diff}
\lim_{h_1 \in \cH, ~\|h_1\|_\cH \to 0} |R(h + h_1) - R(h) - A(h_1)|/\|h_1\|_\cH \to 0.
\# 
We define $A$ as the Fr\'echet derivative of $R$ at $h \in \cH$, and write
\#\label{eq::D_operator}
D_h R(\cdot)= A(\cdot).
\#
\end{definition}
In what follows, we assume that the hypothesis class $\cH$ with the $L_2(\rho)$-inner product is a Hilbert space, where $\rho$ is a distribution over $\cX$. Thus, following from the definition of the Fr\'echet derivative in Definition \ref{def::fre_diff} and the Rieze representation theorem \citep{rudin2006real}, it holds for an $a_{h} \in \cH$ that
\#\label{eq::def_Hil}
D_h R(\cdot) = A(\cdot) = \langle \cdot, a_{h} \rangle_{\cH},
\#
Here we denote by $\langle f, g \rangle_\cH = \int_\cX f(x)\cdot g(x) \ud \rho$ the $L_2(\rho)$-inner product. In what follows, we write
\#\label{eq::def_sl_funcdiff}
(\delta R/\delta h)  (x) = a_{h}(x), \quad \forall x \in \cX, h \in \cH.
\#
We refer to \S\ref{sec::sl_squaredloss} for an example of the Fr\'echet derivative defined in \eqref{eq::def_sl_funcdiff}. We assume that $\cH$ contains the parameterized hypothesis class $\cH_\theta$ defined in \eqref{eq::def_H_theta}, and impose the following assumption on the convexity and the Fr\'echet differentiability of the risk $R_i$ in \eqref{eq::sl_task}.
\begin{assumption}[Convex and Differentiable Risk]
\label{asu::diff_risk}
We assume for all $i\in[n]$ that the risk $R_i$ defined in \eqref{eq::sl_task} is convex and Fr\'echet differentiable on $\cH$.
\end{assumption}
Assumption \ref{asu::diff_risk} is a mild regularity condition on the risk $R_i$, which holds for the risks induced by commonly used loss functions, such as the squared loss and the cross entropy loss. Specifically, the convexity of $R_i$ holds if the loss function $\ell(h, z)$ is convex in $h\in\cH$ for all $z \in \cZ$ \citep{rockafellar1968integrals}. 

The following proposition holds under Assumption \ref{asu::diff_risk}.

\begin{proposition}[Convex and Differentiable Risk \citep{ekeland1999convex}]
\label{prop::conv}
Under Assumption \ref{asu::diff_risk}, it holds for all $i\in[n]$ that
\$
R_i(h_1) \geq R_i(h_2) + \langle \delta R_i/\delta h_2,  h_1 - h_2 \rangle_\cH, \quad\forall h_1, h_2 \in\cH.
\$
\end{proposition}
\begin{proof}
See \cite{ekeland1999convex} for a detailed proof.
\end{proof}

We highlight that the convexity of the risks over $h\in\cH$ does not imply the convexity of the meta-objective defined in \eqref{eq::def_sl_meta_obj}. In contrast, Proposition \ref{prop::conv} characterizes the functional geometry of the risk $R_i$ in the Hilbert space $\cH$ for all $i\in[n]$, which allows us to analyze the global optimality of meta-SL in the sequel.                                                                                                                                                                                                                                                                                                                                                                                                                                                                                                                                                                             

\subsection{Theoretical Results}
\label{sec::sl_opt}
In this section, we characterize the global optimality of the $\epsilon$-stationary point attained by meta-SL defined in \eqref{eq::def_maml_sl_alg}. Let $\theta^*$ be a global minimizer of the meta-objective $L(\theta)$ defined in \eqref{eq::def_sl_meta_obj}, and $\omega$ be the $\epsilon$-stationary point attained by meta-SL such that
\#\label{eq::sl_def_stationary}
\nabla_\omega L(\omega)^\top v \leq \epsilon, \quad\forall v \in \cB = \{\theta \in \RR^d: \|\theta\|_2 \leq 1\}.
\#
Our goal is to upper bound the optimality gap $L(\omega) - L(\theta^*)$. To this end, we first define the mixed distribution $\cM$ over all the distributions $\{\cD_i\}_{i\in[n]}$ as follows,
\#\label{eq::def_mix_sl}
\cM(x, y) = \frac{1}{n} \cdot \sum^n_{i = 1} \cD_i(x, y), \quad \forall (x, y)\in\cX\times\cY.
\#
To simplify the notation, we write $\omega_i$ and $\theta^*_i$ as the parameters that correspond to the outputs of the algorithms $\mA_{\omega}(\cD_i, \ell, \cH)$ and $\mA_{\theta^*}(\cD_i, \ell, \cH)$, respectively. More specifically, according to \eqref{eq::def_sl_alg}, we have
\#\label{eq::pf_sl_eq2}
\omega_i = \omega  - \eta \cdot \nabla_\omega R_i(h_{\omega}), \quad \theta^*_i = \theta^*  - \eta \cdot \nabla_{\theta^*} R_i(h_{\theta^*}), \quad \forall i \in [n],
\# 
where $\eta$ is the learning rate of the algorithms $\mA_{\omega}(\cD_i, \ell, \cH)$ and $\mA_{\theta^*}(\cD_i, \ell, \cH)$.

The following theorem  characterizes the optimality gap of the $\epsilon$-stationary point $\omega$ attained by meta-SL. We define for all $(x, y, x')\in\cX\times\cY\times\cX$ that
\#
\label{eq::def_sl_w} w(x, y, x') &= \frac{1}{n}\cdot \sum^n_{i = 1}  (\delta R_i/\delta h_{\omega_i})(x')\cdot  (\ud \cD_i/\ud \cM)(x, y),\\
\label{eq::def_sl_u} u(x,y, x') &= \biggl(\frac{1}{n}\cdot \sum^n_{i = 1} (\delta R_i/\delta h_{\omega_i})(x')\cdot  \bigl(h_{\omega_i}(x') - h_{\theta^*_i}(x') \bigr)\biggr) \bigg/w(x, y, x'),\\
\label{eq::def_sl_kerphi} \phi_{\ell, \omega}(x, y, x') &= \Bigl(I_d - \eta\cdot\nabla^2_{\omega}\ell\bigl(\phi(x)^\top\omega, (x, y)\bigr)\Bigr) \phi(x'),
\#
where $\ud \cD_i/\ud \cM$ is the Radon-Nikodym derivative and $\delta R_i/\delta h_{\omega_i}$ is the Fr\'echet derivative defined in \eqref{eq::def_sl_funcdiff}. 

\begin{theorem}[Optimality Gap of $\epsilon$-Stationary Point]
\label{thm::sl_opt}
Let $\theta^*$ be a global minimizer of $L(\theta)$. Also, let $\omega$ be the $\epsilon$-stationary point defined in \eqref{eq::sl_def_stationary}. Let $\ell(h_\theta(x), (x, y))$ be twice differentiable with respect to all $\theta\in\R^d$ and $(x, y)\in\cX\times\cY$. Under Assumption \ref{asu::diff_risk}, it holds for all $R > 0$ that
\#\label{eq::sl_opt_bound}
L(\omega) - L(\theta^*)\leq \underbrace{R\cdot \epsilon}_{\textstyle{\rm (i)}} + \underbrace{\| w \|_{\cM\cdot\rho}}_{\textstyle{\rm (ii)}}\cdot \underbrace{\inf_{v\in\cB_R}\|u(\cdot) - \phi_{\ell, \omega}(\cdot)^\top v\|_{\cM\cdot\rho}}_{\textstyle{\rm (iii)}},
\#
where we define $\cB_R = \{\theta \in \RR^d: \|\theta\|_2 \leq R\}$ as the ball with radius $R$ and \$\|w\|_{\cM\cdot \rho} = \biggl(\int w^2(x, y, x') \ud \cM(x, y)\ud \rho(x')\biggr)^{1/2}\$ as the $L_2(\cM\cdot\rho)$-norm of $w$. 
\end{theorem}
\begin{proof}
See \S\ref{pf::sl_opt} for a detailed proof.
\end{proof}

By Theorem \ref{thm::sl_opt}, the optimality gap of the $\epsilon$-stationary point $\omega$ hinges on the three terms on the right-hand side of \eqref{eq::sl_opt_bound}.  
Here term (i) characterizes the deviation of the $\epsilon$-stationary point $\omega$ from a stationary point. Term (ii) characterizes the difficulty of all the subtasks sampled from the task distribution $\iota$. Specifically, given the $\epsilon$-stationary point $\omega$, if the output $h_{\omega_i}$ of $\mA_\omega(\cD_i, \ell, \cH)$ well approximates the minimizer of the risk $R_i$ in \eqref{eq::sl_task}, then the Fr\'echet derivative $\delta R_i/\delta h_{\omega_i}$ defined in \eqref{eq::def_sl_funcdiff} is close to zero. Meanwhile, the Radon-Nikodym derivative $\ud \cD_i/\ud \cM$ characterizes the deviation of the distribution $\cD_i$ from the mixed distribution $\cM$ defined in \eqref{eq::def_mix_sl}, which is upper bounded if $\cD_i$ is close to $\cM$. Thus, term (ii) is upper bounded if $h_{\omega_i}$ well approximates the minimizer of $R_i$ and $\cD_i$ is close to $\cM$ for all $i\in[n]$. Term (iii) characterizes the representation power of the feature mapping $\phi_{\ell, \omega}$ defined in \eqref{eq::def_sl_kerphi}. 
Specifically, if the function $u$ defined in \eqref{eq::def_sl_u} of Theorem \ref{thm::sl_opt} is well approximated by $\phi_{\ell, \omega}(\cdot)^\top v$ for some $v\in\cB_R$, then term (iii) is small. In conclusion, if the subtasks generated by the task distribution $\iota$ are sufficiently regular so that term (ii) is upper bounded, and the linear class $\{\phi_{\ell, \omega}(\cdot)^\top v: v\in\cB_R\}$ has sufficient representation power, then $\omega$ is approximately globally optimal. 
See \S\ref{sec::sl_squaredloss} for a corollary of Theorem \ref{thm::sl_opt} when it is adapted to the squared loss.

\section{Neural Network Prameterization}
In this section, we present the global optimality analysis of meta-RL and meta-SL with the overparameterized two-layer neural network parameterization, namely neural meta-RL and neural meta-SL, respectively. Specifically, for both neural meta-RL and neural meta-SL, we show that the global optimality of the attained $\epsilon$-stationary points hinges on the representation power of the corresponding classes of overparameterized two-layer neural networks.

\subsection{Neural Network}
We first introduce the neural network parameterization. For $x \in \RR^d$, $b = (b_1, \ldots, b_m)\in\RR^m$, and $W = ([W]^\top_1, \ldots, [W]^\top_m) \in \RR^{md}$, we define
\#\label{eq::def_nn}
f(x; b, W) = \frac{1}{\sqrt{m}}\cdot\sum^m_{r = 1}b_r\cdot \sigma\bigl( [W]_r^\top x\bigr),
\#
where $\sigma(x) = x\cdot \ind\{x > 0\}$ is the rectified linear unit (ReLU).

We set $m$ to be divisible by two and initialize the parameter $W$ with  $[W_{\text{\rm init}}]_r = [W_{\text{\rm init}}]_{r+m/2} \sim N(0, I_d/d)$ for $r \in [m/2]$. Meanwhile, we initialize $b_r = 1$ and $b_{r + m/2} = -1$ for all $r \in [m/2]$. Such initialization \citep{bai2019beyond} is almost equivalent to the independent and identical initialization of $[W_{\text{\rm init}}]_r$ for $r\in[m]$ in our analysis, and ensures that $f(x; W_{\text{\rm init}}) = 0$ for all $x\in\cX$. In what follows, we fix $b_r$ for all $r\in[m]$ and only optimize over $W$. We write $f(x; W) = f(x; b, W)$ in the sequel for notational simplicity. Note that $f(x; W)$ is almost everywhere differentiable with respect to $W$, and it holds for all $x \in \RR^d$ and $W\in\RR^{md}$ that $\nabla_W f(x; W) = ([\nabla_W f(x; W)]^\top_1, \ldots, [\nabla_W f(x; W)]^\top_m)^\top \in \RR^{md}$, where
\#\label{eq::def_nn_phi}
\bigl[\nabla_W f(x; W)\bigr]_r =\bigl[\phi_W (x)\bigr]_r = \frac{b_r}{\sqrt{m}}\cdot x\cdot \ind\bigl\{ [W]_r^\top x > 0\bigr\}, \quad \forall r \in [m].
\#
Here we define the feature mapping as $\phi_W (x) = ([\phi_W(x)]^\top_1, \ldots, [\phi_W(x)]^\top_m)$ for all $x \in \RR^d$ and $W\in\RR^{md}$. It then follows from the definition of $f(x; W)$ in \eqref{eq::def_nn} that $f(x; W) = \phi_W(x)^\top W$. In the sequel, we denote by $\EE_{\text{\rm init}}$ the expectation with respect to the random initialization of the neural network.
\subsection{Neural Meta-RL}
In this section, we analyze the global optimality of the $\epsilon$-stationary point attained by meta-RL when the main effect $\pi_\theta$ is parameterized by the neural network defined in \eqref{eq::def_nn}. Without loss of generality, we assume that $\cS\times\cA \subseteq \RR^d$ and $\|(s, a)\|_2\leq 1$ for all $(s, a)\in\cS\times\cA$. Similar to \eqref{eq::def_leader}, we parameterize the main effect $\pi_\theta$ as follows,
\#\label{eq::def_leader_nn}
\pi_\theta(a\given s) = \frac{\exp\bigl\{1/\tau\cdot f\bigl((s, a); \theta\bigr) \bigr\}}{\sum_{a'\in\cA}\exp\bigl\{1/\tau\cdot f\bigl((s, a'); \theta\bigr) \bigr\}}, \quad \forall (s, a)\in\cS\times\cA,
\#
where $f(\cdot; \theta)$ is the neural network defined in \eqref{eq::def_nn} with $W=\theta$ for all $\theta\in\RR^{md}$. Correspondingly, given the MDP $(\cS, \cA, P_i, r_i, \gamma_i, \zeta_i)$, the maximizer $\pi_{i, \theta}$ defined in \eqref{eq::def_PPO} takes the form of
\#\label{eq::PPO_opt_nn}
\pi_{i, \theta}(\cdot\given s) \propto \exp\Bigl(1/\tau\cdot f\bigl((s, \cdot);\theta\bigr) +\eta\cdot Q^{\pi_\theta}_i(s, \cdot)\Bigr), \quad \forall s \in \cS,
\#
where $Q^{\pi_\theta}_i$ is the action-value function of $\pi_\theta$ corresponding the MDP $(\cS, \cA, P_i, r_i, \gamma_i, \zeta_i)$. Neural meta-RL maximizes the following meta-objective via gradient ascent with $W_{\text{\rm init}}$ as the starting point,
\#\label{eq::def_meta_obj_rl_nn}
L(\theta) = \frac{1}{n}\cdot \sum^n_{i =1} J_i(\pi_{i, \theta}),
\#
where $\pi_{i, \theta}$ is defined in \eqref{eq::PPO_opt_nn}, and $J_i(\pi_{i, \theta})$ is the expected total reward of $\pi_{i, \theta}$ corresponding to the MDP $(\cS, \cA, P_i, r_i, \gamma_i, \zeta_i)$. In what follows, we analyze the global optimality of the $\epsilon$-stationary point $\omega$ of the meta-objective $L$ attained by neural meta-RL. Specifically, we define $\omega$ as follows,
\#\label{eq::eps_optimal}
\nabla_\omega L(\omega)^\top (v - \omega) \leq \epsilon, \quad\forall v \in \cB_{\text{\rm init}} = \{\theta \in \RR^d: \|\theta - W_{\text{\rm init}}\|_2 \leq R_T\}.
\#
Here $W_{\text{\rm init}}$ is the initial parameter, and the radius $R_T$ is the maximum trajectory length of $T$ gradient ascent steps. 

We impose the following regularity condition on the mixed meta-visitation measure $\varrho_{\pi_\theta}$ defined in \eqref{eq::def_mixed_meta} of Definition \ref{def::meta_meta_visit}.
\begin{assumption}[Regularity Condition on $\varrho_{\pi_\theta}$]
\label{asu::reg_cond_rl}
We assume for all $\theta \in \RR^{md}$ that
\$
\EE_{(s, a)\sim \varrho_{\pi_\theta}}\Bigl[\ind\bigl\{|y^\top (s, a)|\leq u\bigr\}\Bigr]\leq c\cdot u/\|y\|_2,\quad \forall y \in \RR^d, u >0,
\$
where $c >0$ is an absolute constant.
\end{assumption}
Assumption \ref{asu::reg_cond_rl} is imposed to rule out the corner case where $\varrho_{\pi_\theta}$ has a point mass at a specific state action pair $(s, a)\in\cS\times\cA$. Similar assumptions arise in the analysis of RL with neural network parameterization \citep{cai2019neural, liu2019neural}.

The following corollary characterizes the optimality gap of the $\epsilon$-stationary point defined in \eqref{eq::eps_optimal}. Let $\theta^*$ be a global maximizer of the meta-objective $L(\theta)$ defined in \eqref{eq::def_meta_obj_rl_nn}. We define
\#\label{eq::def_c_nn}
c_\omega(s, a) = f\bigl((s, a); \omega\bigr)+ f_\omega(s, a), \quad \forall (s, a)\in\cS\times\cA,
\#
where $f(\cdot; \omega)$ is the neural network defined in \eqref{eq::def_nn} with $W = \omega$ and $f_{ \omega}$ is defined in \eqref{eq::def_f}. 

\begin{corollary}[Optimality Gap of $\epsilon$-Stationary Point]
\label{cor::eps_nn_opt}
Under Assumptions \ref{asu::concen_coeff} and \ref{asu::reg_cond_rl}, for the $\epsilon$-stationary point $\omega$ defined in \eqref{eq::eps_optimal}, we have
\#\label{eq::nn_opt_rl}
\EE_{\text{\rm init}}\bigl[L(\theta^*) - L(\omega)\bigr] &\leq  \underbrace{\epsilon}_{\textstyle{\rm (i)}} + \underbrace{C\cdot \EE_{\text{\rm init}}\Bigl[ \inf_{v\in\cB_{\text{\rm init}}} \bigl\|c_{\omega}(\cdot, \cdot) - f\bigl((\cdot, \cdot); v\bigr)\bigr\|_{\varrho_{\pi_\omega}}\Bigr]}_{\textstyle{\rm (ii)}} + \underbrace{\cO(R_T^{3/2}\cdot m^{-1/4})}_{\textstyle{\rm (iii)}},
\#
where $C = 2C_0\cdot Q_{\max}/\tau\cdot(1 + 2Q_{\max}\cdot\overline\gamma\cdot\eta)$, $\overline \gamma = (\sum^n_{i = 1}\gamma_i)/n$, $C_0$ is defined in Assumption \ref{asu::concen_coeff}, and $\cB_{\text{\rm init}}$ is the parameter space defined in \eqref{eq::eps_optimal}.
\end{corollary}
\begin{proof}
See \S\ref{pf::eps_nn_opt} for a detailed proof.
\end{proof}

By Corollary \ref{cor::eps_nn_opt}, the global optimality of the $\epsilon$-stationary point $\omega$ is upper bounded by the three terms on the right-hand side of \eqref{eq::nn_opt_rl}. Here term (i) characterizes the deviation of $\omega$ from a stationary point. Term (ii) characterizes the representation power of neural networks. Specifically, if the function $c_{\omega}$ defined in \eqref{eq::def_c_nn} is well approximated by the neural network defined in \eqref{eq::def_nn} with a parameter from the parameter space $\cB_{\text{\rm init}}$, then term (ii) is small. Term (iii) is the linearization error of the neural networks, which characterizes the deviation of a neural network from its first-order Taylor expansion at the initial parameter $W_{\text{\rm init}}$. Such an error is small for a sufficiently large width $m$, that is, if the neural network is overparameterized. In conclusion, if the class of overparameterized two-layer neural networks with the parameter space $\cB_{\text{\rm init}}$ has sufficient representation power, then the $\epsilon$-stationary point $\omega$ attained by neural meta-RL is approximately globally optimal.

\subsection{Neural Meta-SL}
In this section, we analyze the global optimality of the $\epsilon$-stationary point attained by neural meta-SL associated with the squared loss, where we parameterize the hypothesis $h_\theta(\cdot) = f(\cdot; \theta)$ by the neural network defined in \eqref{eq::def_nn}. Specifically, neural meta-SL minimizes the meta-objective $L$ defined in \eqref{eq::def_sl_meta_obj} via gradient descent defined in \eqref{eq::def_maml_sl_alg} with $W_{\text{\rm init}}$ as the starting point. We analyze the global optimality of the $\epsilon$-stationary point $\omega$ attained by neural meta-SL, which is defined as follows,
\#\label{eq::eps_def_sl}
\nabla_\omega L(\omega)^\top ( \omega - v) \leq \epsilon, \quad\forall v  \in \cB_{\text{\rm init}} = \{\theta \in \RR^d: \|\theta - W_{\text{\rm init}}\|_2 \leq R_T\}.
\#
Here $R_T$ is the maximum trajectory length of $T$ gradient descent steps. In what follows, we set $\cX = \{x \in \RR^d: \|x\|_2 \leq 1\}$. Similar to Assumption \ref{asu::reg_cond_rl}, we impose the following regularity condition on the distribution $\rho$ that defines the Hilbert space in \eqref{eq::def_Hil}.

\begin{assumption}[Regularity Condition on $\rho$]
\label{asu::reg_cond_sl}
We assume for an absolute constant $c >0$ that
\$
\EE_{x\sim \rho}\Bigl[\ind\bigl\{|x^\top y| \leq u\bigr\}\Bigr] \leq c\cdot u/\|y\|_2, \quad \forall y \in \RR^d, u >0.
\$
\end{assumption}

Such an assumption holds if the probability density function of $\rho$ is upper bounded by an absolute constant. Under Assumption \ref{asu::reg_cond_sl}, the following corollary characterizes the optimality gap of the $\epsilon$-stationary point $\omega$ defined in \eqref{eq::eps_def_sl}. We define
\#
\label{eq::def_sl_nn_B}K_{\omega, \eta} &= \EE_{x\sim \rho}\bigl[I_{md} - 2\eta\cdot \phi_\omega(x)\phi_\omega(x)^\top\bigr], \quad\cB_0 = K_{\omega, \eta}(\omega - \cB_{\text{\rm init}}) + W_{\text{\rm init}},\\
\label{eq::def_sl_nn_u}\overline u(x) &= f(x;W_{\text{\rm init}}) +  \biggl( \sum^n_{i = 1} (\delta R_i/\delta h_{\omega_i})(x)\cdot  \bigl(h_{\omega_i}(x) - h_{\theta^*_i}(x) \bigr)\biggr) \bigg/\biggl(\sum^n_{i = 1}  (\delta R_i/\delta h_{\omega_i})(x)\biggr),
\#
where $f(\cdot;W_{\text{\rm init}})$ and $\phi_\omega$ are the neural network and the feature mapping defined in \eqref{eq::def_nn} with $W = W_{\text{\rm init}}$ and \eqref{eq::def_nn_phi} with $W = \omega$, respectively, $\cB_{\text{\rm init}}$ is the parameter space defined in \eqref{eq::eps_def_sl}, $W_{\text{\rm init}}$ is the initial parameter, and $\omega_i$, $\theta_i^*$ are the parameters defined in \eqref{eq::pf_sl_eq2}. We further define the average risk $\overline R$ as follows,
\#\label{eq::def_sl_sqloss_R11}
\overline R  = \frac{1}{n}\cdot\sum^n_{i = 1} R_i^{1/2}(h_{\omega_i})=\frac{1}{n}\cdot\sum^n_{i = 1}\Bigl\{\EE_{(x, y) \sim \cD_i}\Bigl[\bigl(y - h_{\omega_i}(x)\bigr)^2\Bigr]\Bigr\}^{1/2}.
\# 
\begin{corollary}[Optimality Gap of $\epsilon$-Stationary Point]
\label{cor::nnsl}
We denote by $\overline\cD_i$ the marginal distribution of $\cD_i$ over $\cX$. Let $\overline\cD_i = \rho$ for all $i \in [n]$ and $|y| \leq Y_{\max}$ for all $y\in\cY$. Under Assumptions \ref{asu::diff_risk} and \ref{asu::reg_cond_sl}, for the squared loss $\ell(h, (x, y)) = (h(x) - y)^2$ and $\omega$ defined in \eqref{eq::eps_def_sl}, we have
\#\label{eq::opt_nn_sl}
\EE_{\text{\rm init}}\bigl[L(\omega) - L(\theta^*) \bigr]\leq \epsilon + \EE_{\rm init} \Bigl[2\overline R \cdot \inf_{v\in\cB_{0}} \|\overline u(\cdot) - f(\cdot; v)\|_{\rho} \Bigr]+  \cO(G_T^{3/2}\cdot m^{-1/4}),
\#
where $G_T = (1  + \eta)\cdot R_T + \eta\cdot Y_{\max}$, $R_T$ is the maximum trajectory length in \eqref{eq::eps_def_sl}, and $\eta$ is the learning rate of $\mA_{\omega}$ in \eqref{eq::pf_sl_eq2}.
\end{corollary}
\begin{proof}
See \S\ref{pf::nnsl} for a detailed proof.
\end{proof}


Similar to Corollary \ref{cor::eps_nn_opt}, by Corollary \ref{cor::nnsl}, if the function $\overline u$ defined in \eqref{eq::def_sl_nn_u} is well approximated by an overparameterized two-layer neural network with a parameter from the parameter space $\cB_0$ defined in \eqref{eq::def_sl_nn_B}, and the average risk $\overline R$ defined in \eqref{eq::def_sl_sqloss_R11} is upper bounded, then the $\epsilon$-stationary point $\omega$ attained by neural meta-SL is approximately globally optimal.

\bibliographystyle{ims}
\bibliography{MAML}

\newpage
\appendix
\section{Meta-SL with Squared Loss}
\label{sec::sl_squaredloss}
In this section, we analyze the global optimality of meta-SL with the squared loss. The optimality gap characterized in Theorem \ref{thm::sl_opt} has a more straightforward interpretation when restricted to meta-SL with the squared loss, which is defined as
\#\label{eq::def_sl_sqloss}
\ell\bigl(h, (x, y)\bigr) = \bigl(h(x) - y\bigr)^2, \quad \forall h\in\cH, (x, y)\in \cX\times\cY.
\#
The following proposition calculates the Fr\'echet derivative $\delta R_i/\delta h$ defined in \eqref{eq::def_sl_funcdiff} for the squared loss.
\begin{proposition}
\label{prop::sqloss_diff}
We denote by $\overline\cD_i$ the marginal distribution of $\cD_i$ over $\cX$. Let $\overline\cD_i = \rho$ for all $i \in [n]$. For the squared loss $\ell$ defined in \eqref{eq::def_sl_sqloss} and $R_i = \EE_{(x, y)\sim \cD_i}[\ell(h, (x, y))]$, it holds that
\#\label{eq::sqloss_diff}
(\delta R_i /\delta h)(x') = 2\EE_{(x, y)\sim \cD_i}\bigl[h(x)- y~\big | ~ x = x'\bigr], \quad \forall h \in \cH,x'\in\cX.
\#
\end{proposition}
\begin{proof}
See \S\ref{pf::sqloss_diff} for a detailed proof.
\end{proof}

By Proposition \ref{prop::sqloss_diff}, we obtain from Jensen's inequality that
\$
\|\delta R_i/\delta h_{\omega_i}\|^2_\rho \leq 4\EE_{(x, y)\sim \cD_i} \Bigl[\bigl(h_{\omega_i}(x) - y \bigr)^2\Bigr] = 4R_i(h_{\omega_i}).
\$
Meanwhile, recall that the function $w$ defined in \eqref{eq::def_sl_w} is a weighted average over the Fr\'echet derivatives $\{\delta R_i/\delta h_{\omega_i}\}_{i\in[n]}$. Hence, the $L_2(\rho)$-norm of the function $w$ characterizes the difficulty of subtasks by aggregating the risks $R_i(h_{\omega_i})$.


The following corollary characterizes the the optimality gap of the $\epsilon$-stationary point $\omega$ attained by meta-SL, which is defined in \eqref{eq::sl_def_stationary}. We define
\#
\label{eq::def_sl_sqloss_kernel}K_\eta &= \EE_{x\sim \rho}\bigl[I_d - 2\eta\cdot \phi(x)\phi(x)^\top\bigr],\\
\label{eq::def_sl_sqloss_u}u(x') &= \biggl( \sum^n_{i = 1} (\delta R_i/\delta h_{\omega_i})(x')\cdot  \bigl(h_{\omega_i}(x') - h_{\theta^*_i}(x') \bigr)\biggr) \bigg/\biggl(\sum^n_{i = 1}  (\delta R_i/\delta h_{\omega_i})(x')\biggr),\\
\label{eq::def_sl_sqloss_R}
\overline R  &= \frac{1}{n}\cdot\sum^n_{i = 1} R_i^{1/2}(h_{\omega_i})=\frac{1}{n}\sum^n_{i = 1}\Bigl\{\EE_{(x, y) \sim \cD_i}\Bigl[\bigl(y - h_{\omega_i}(x)\bigr)^2\Bigr]\Bigr\}^{1/2},
\#
where $\omega_i$ and $\theta_i^*$ are the parameters defined in \eqref{eq::pf_sl_eq2}, and $\eta$ is the learning rate of $\mA_\omega$ in \eqref{eq::pf_sl_eq2}.

\begin{corollary}
\label{cor::opt_sl_sqloss}
We denote by $\overline\cD_i$ the marginal distribution of $\cD_i$ over $\cX$. Let $\overline\cD_i = \rho$ for all $i \in [n]$. Under Assumption \ref{asu::diff_risk}, for the squared loss $\ell$ defined in \eqref{eq::def_sl_sqloss} and $R >0$, we have
\$
L(\omega) - L(\theta^*) \leq R\cdot\epsilon+ 2\overline R \cdot \inf_{v\in\cB} \|u - (K_\eta\cdot\phi)^\top (R\cdot v)\|_{\rho}.
\$
\end{corollary}
\begin{proof}
See \S\ref{pf::opt_sl_sqloss} for a detailed proof.
\end{proof}
By Corollary \ref{cor::opt_sl_sqloss}, the optimality gap $L(\omega) - L(\theta^*)$ hinges on the average risk $\overline R$, the representation power of the feature $\phi$, and the kernel $K_\eta$ defined in \eqref{eq::def_sl_sqloss_kernel}. Note that
\$
\inf_{v\in\cB} \|u - (K_\eta\cdot\phi)^\top (R\cdot v)\|_{\rho} = \inf_{v\in R\cdot K_\eta\cdot\cB} \|u - \phi^\top  v\|_{\rho},
\$
where we write $R\cdot K_\eta\cdot \cB = \{v\in\RR^d: v = R\cdot K_\eta\cdot u, u \in \cB\}$. Hence, if $\phi(\cdot)^\top \theta$ well approximates the function $u$ defined in \eqref{eq::def_sl_sqloss_u} for a parameter $\theta\in R\cdot K_\eta\cdot \cB$ and $\overline R$ is upper bounded, then the $\epsilon$-stationary point $\omega$ attained by meta-SL is approximately globally optimal.

\section{Proof of Main Result}
In this section, we present the proofs of the main results.
\subsection{Proof of Theorem \ref{thm::eps_opt}}
\label{pf::thm_eps_opt}
\begin{proof}
By Lemma \ref{lem::performance_diff}, it holds for all $i\in[n]$ that
\#\label{eq::pf_eps_opt_eq4}
J_i(\pi_{i, \theta^*}) - J_i(\pi_{i, \omega}) =  (1 - \gamma	_i)^{-1}\cdot \EE_{(s, a) \sim \sigma_{ \pi_{i, \theta^*}}}\bigl[A^{\pi_{i, \omega}}_i(s, a)\bigr],
\#
where $A^{\pi_{i, \omega}}_i$ and $\sigma_{\pi_{i, \theta^*}}$ are the advantage function and the state-action visitation measure of the policies $\pi_{i, \theta^*}$, $\pi_{i, \omega}$, respectively, corresponding to the MDP $(\cS, \cA, P_i, r_i, \gamma_i, \zeta_i)$. Meanwhile, note that
\#\label{eq::pf_eps_opt_eq5}
\nabla_\omega L(\omega)^\top v \leq \epsilon,\quad\forall v\in\RR^d, \|v\|_2 \leq 1.
\# 
Thus, combining \eqref{eq::pf_eps_opt_eq4} and \eqref{eq::pf_eps_opt_eq5}, we obtain for all $ v \in \cB = \{\theta \in \RR^d: \|\theta\|_2 \leq 1\}$ and $R>0$ that
\#\label{eq::pf_eps_opt_eq6}
L(\theta^*) - L(\omega) &\leq R\cdot\epsilon - R\cdot \nabla_\omega L(\omega)^\top v + \frac{1}{n}\cdot \sum^n_{i = 1}J_i(\pi_{i, \theta^*}) - J_i(\pi_{i, \omega}) \notag\\
&= R\cdot\epsilon - R\cdot \nabla_\omega L(\omega)^\top v +   \frac{1}{n}\cdot\sum^n_{i = 1}(1 - \gamma	_i)^{-1}\cdot \EE_{(s, a) \sim \sigma_{\pi_{i, \theta^*}}}\bigl[A^{\pi_{i, \omega}}_i(s, a)\bigr].
\#
In what follows, we upper bound the right-hand side of \eqref{eq::pf_eps_opt_eq6}. By Proposition \ref{prop::meta_grad}, we have the following lemma that calculates $\nabla_\omega L(\omega)$.

\begin{lemma}
\label{lem::meta_grad_refine}
It holds for all $\theta \in \RR^d$ that
\#\label{eq::meta_gra_refine}
\nabla_\theta L(\theta) = \frac{1}{n}\cdot \sum^n_{i = 1}\EE_{(s', a')\sim \varsigma_{i, \pi_\theta}}\bigl[g_{i, \theta}(s', a')\cdot \phi(s', a')\bigr],
\#
where $\varsigma_{i, \pi_\theta}$ is the meta-visitation measure defined in \eqref{eq::def_meta_visit} of Definition \ref{def::meta_meta_visit}, and 
\#\label{eq::def_g}
g_{i, \theta}(s', a') = 1/\tau \cdot A^{\pi_{i, \theta}}_i(s', a') \cdot (\ud \sigma_{\pi_{i, \theta}}/\ud \varsigma_{i, \pi_\theta})(s', a') + \gamma_i\cdot\eta/\tau\cdot G_{i, \pi_\theta}(s', a')\cdot A^{\pi_\theta}_i(s', a').
\#
Here $A^{\pi_{i, \theta}}_i$ and $A^{\pi_\theta}_i$ are the advantage functions of $\pi_{i, \theta}$ and $\pi_\theta$, respectively, corresponding to the MDP $(\cS, \cA, P_i, r_i, \gamma_i, \zeta_i)$, and $G_{i, \pi_\theta}$ is defined as follows,
\#\label{eq::meta_grad_def_A}
G_{i, \pi_\theta}(s', a') = \EE_{(s',a',s,a)\sim\rho_{i, \pi_\theta}}\bigl[A^{\pi_{i, \theta}}_i(s, a) \,\big|\, s', a'\bigr],
\#
where $\rho_{i, \pi_\theta}$ is the joint meta-visitation measure defined in \eqref{eq::def_meta_joint} of Definition \ref{def::meta_meta_visit}.
\end{lemma}
\begin{proof}
See \S\ref{pf::meta_grad_refine} for a detailed proof.
\end{proof}

By Lemma \ref{lem::meta_grad_refine}, we obtain for all  $v \in \cB = \{\theta \in \RR^d: \|\theta\|_2 \leq 1\}$ that
\#\label{eq::pf_eps_opt_eq7}
\nabla_\omega L(\omega)^\top v = \frac{1}{n}\cdot \sum^n_{i = 1} \EE_{(s', a')\sim \varsigma_{i, \pi_{\omega}}}\bigl[g_{i, \omega}(s', a') \cdot  \phi(s', a')^\top v\bigr],
\#
where $g_{i, \omega}$ is defined in \eqref{eq::def_g} of Lemma \ref{lem::meta_grad_refine} with $\theta = \omega$. By plugging \eqref{eq::pf_eps_opt_eq7} into \eqref{eq::pf_eps_opt_eq6}, we obtain that
\#\label{eq::pf_eps_opt_eq8}
L(\theta^*) - L(\omega) &\leq R\cdot\epsilon + \frac{1}{n}\cdot \sum^n_{i = 1}(1 - \gamma	_i)^{-1}\cdot \EE_{(s, a) \sim \sigma_{\pi_{i, \theta^*}}}\bigl[A^{\pi_{i, \omega}}_i(s, a)\bigr] \notag\\
& \qquad\qquad\qquad\qquad- R\cdot \EE_{(s', a')\sim \varsigma_{i, \pi_{\omega}}}\bigl[g_{i, \omega}(s', a') \cdot  \phi(s', a')^\top v\bigr]\notag\\
&\leq R\cdot\epsilon+\EE_{(s', a')\sim \varrho_{\pi_\omega}}\biggl[\frac{1}{n}\cdot\sum^n_{i = 1}(1 - \gamma	_i)^{-1}\cdot A^{\pi_{i, \omega}}_i(s', a')\cdot \frac{\ud \sigma_{\pi_{i, \theta^*}}}{ \ud \varrho_{\pi_\omega}}(s', a')\\
&\qquad\qquad\qquad\qquad\qquad  - \frac{1}{n} \cdot\sum^n_{i = 1} g_{i, \omega}(s', a') \cdot \frac{\ud \varsigma_{i, \pi_{\omega}}}{ \ud\varrho_{\pi_\omega}}(s', a')\cdot \phi(s', a')^\top (R\cdot v) \biggr],\notag
\#
where $\ud \sigma_{i, \pi_{\theta^*}}/ \ud\varrho_{\pi_\omega}$ and $\ud \varsigma_{i, \pi_{\omega}} / \ud\varrho_{\pi_\omega}$ are Radon-Nikodym derivatives, and $\varrho_{\pi_\omega}$ is the mixed meta-visitation measure defined in \eqref{eq::def_mixed_meta} of Definition \ref{def::meta_meta_visit}. By the Cauchy-Schwartz inequality, we obtain from \eqref{eq::pf_eps_opt_eq8} that
\#\label{eq::pf_eps_opt_eq9}
L(\theta^*) - L(\omega) &\leq R\cdot\epsilon + \underbrace{ \biggl\|\frac{1}{n} \cdot\sum^n_{i = 1} g_{i, \omega} \cdot \frac{\ud \varsigma_{i, \pi_{\omega}}}{ \ud \varrho_{\pi_\omega}} \biggr\|_{\varrho_{\pi_\omega}}}_{\textstyle{H}}\cdot \|f_{\omega}(\cdot, \cdot) - \phi(\cdot, \cdot)^\top (R\cdot v)\|_{\varrho_{\pi_\omega}},
\#
where
\#\label{eq::pf_eps_opt_def_f}
f_\omega(s', a') = \biggl(\frac{1}{n}\cdot\sum^n_{i = 1}\frac{A^{\pi_{i, \omega}}_i(s', a')}{1 - \gamma	_i}\cdot \frac{\ud \sigma_{\pi_{i, \theta^*}}}{ \ud\varrho_{\pi_\omega}}(s', a')\biggr)\bigg/\biggl(\frac{1}{n} \cdot\sum^n_{i = 1} g_{i, \omega}(s', a') \cdot \frac{\ud \varsigma_{i, \pi_{\omega}}}{ \ud \varrho_{\pi_\omega}}(s', a')\biggr).
\#

It remains to upper bound the norm $H$ in \eqref{eq::pf_eps_opt_eq9}. By the definition of $g_{i, \omega}$ in \eqref{eq::def_g} of Lemma \ref{lem::meta_grad_refine}, we have
\#\label{eq::pf_eps_opt_eq10}
g_{i, \omega}(s', a') \cdot \frac{\ud \varsigma_{i, \pi_{\omega}}}{ \ud\varrho_{\pi_\omega}}(s', a') &= 1/\tau \cdot A^{\pi_{i, \omega}}_i(s', a') \cdot \frac{\sigma_{\pi_{i, \omega}}}{\varrho_{\pi_\omega}}(s', a') \\
&\qquad+  \gamma_i\cdot\eta/\tau\cdot A_{i, \omega}(s', a')\cdot A^{\pi_\omega}_i(s', a')\cdot \frac{\ud \varsigma_{i, \pi_{\omega}}}{ \ud \varrho_{\pi_\omega}}(s', a'),\notag
\#
which holds for all $i\in[n]$ and $(s', a')\in\cS\times\cA$. By the assumption that the rewards are upper bounded by $Q_{\max}$, we have
\#\label{eq::pf_eps_opt_eq11}
|A^{\pi_\omega}_i(s, a)|\leq 2Q_{\max}, \quad |A^{\pi_{i, \omega}}_i(s, a)| \leq 2Q_{\max}, \quad \forall  (s, a)\in\cS\times\cA, i \in [n].
\#
Meanwhile, by the definition of $A_{i, \omega}$ in \eqref{eq::meta_grad_def_A} of Lemma \ref{lem::meta_grad_refine}, we further have
\#\label{eq::pf_eps_opt_eq12}
|A_{i, \omega}(s', a')| \leq \EE_{(s',a',s,a)\sim \rho_{i, \pi_\omega}}\bigl[|A^{\pi_{i, \omega}}_i(s, a)| \,\big|\, s', a'\bigr] \leq 2Q_{\max}, \quad \forall(s', a')\in\cS\times\cA, i \in [n].
\#
Combining \eqref{eq::pf_eps_opt_eq10}, \eqref{eq::pf_eps_opt_eq11}, and \eqref{eq::pf_eps_opt_eq12}, we have
\#\label{eq::pf_eps_opt_eq13}
\biggl\|g_{i, \omega} \cdot \frac{\ud \varsigma_{i, \pi_{\omega}}}{ \ud \varrho_{\pi_\omega}}\biggr\|_{\varrho_{\pi_\omega}}\leq 2Q_{\max}/\tau\cdot\biggl\|\frac{\ud\sigma_{\pi_{i, \omega}}}{\ud\varrho_{\pi_\omega}}\biggr\|_{\varrho_{\pi_\omega}} + 4Q_{\max}^2\cdot\gamma_i\cdot\eta/\tau\cdot\biggl\|\frac{\ud \varsigma_{i, \pi_{\omega}}}{ \ud \varrho_{\pi_\omega}}\biggr\|_{\varrho_{\pi_\omega}}, \quad\forall i \in [n].
\#
Thus, following from Assumption \ref{asu::concen_coeff} and \eqref{eq::pf_eps_opt_eq13}, we obtain that
\#\label{eq::pf_eps_opt_eq14}
\biggl\|\frac{1}{n} \cdot\sum^n_{i = 1} g_{i, \omega} \cdot \frac{\ud \varsigma_{i, \pi_{\omega}}}{ \ud \varrho_{\pi_\omega}} \biggr\|_{\varrho_{\pi_\omega}} &\leq \frac{1}{n}\cdot\sum^n_{i = 1} \biggl\|g_{i, \omega} \cdot \frac{\ud \varsigma_{i, \pi_{\omega}}}{ \ud \varrho_{\pi_\omega}}\biggr\|_{\varrho_{\pi_\omega}} \leq 2C_0\cdot Q_{\max}/\tau\cdot(1 + 2Q_{\max}\cdot\overline\gamma\cdot\eta),
\#
where we define $\overline \gamma = (\sum^n_{i = 1}\gamma_i)/n$. Finally, by plugging \eqref{eq::pf_eps_opt_eq14} into \eqref{eq::pf_eps_opt_eq9}, we have for all $v\in\cB$ that
\#\label{eq::pf_eps_opt_eq15}
L(\theta^*) - L(\omega) &\leq R\cdot\epsilon + 2C_0\cdot Q_{\max}/\tau\cdot(1 + 2Q_{\max}\cdot\overline\gamma\cdot\eta)\cdot \|f_{\omega}(\cdot, \cdot) - \phi(\cdot, \cdot)^\top (R\cdot v)\|_{\varrho_{\pi_\omega}},
\#
where $f_\omega$ is defined in \eqref{eq::pf_eps_opt_def_f} and $\overline \gamma = (\sum^n_{i = 1}\gamma_i)/n$. By taking the infimum over $v\in\cB$ on the right-hand side of \eqref{eq::pf_eps_opt_eq15}, we complete the proof of Theorem \ref{thm::eps_opt}.
\end{proof}

\subsection{Proof of Theorem \ref{thm::sl_opt}}
\label{pf::sl_opt}
\begin{proof}
By Assumption \ref{asu::diff_risk} and Proposotion \ref{prop::conv}, we have
\#\label{eq::pf_sl_eq000}
R_i(h_{\theta_2}) - R_i(h_{\theta_1}) \leq \langle \delta R_i/\delta h_{\theta_2},  h_{\theta_2} - h_{\theta_1} \rangle_\cH, \quad \forall i \in [n] , \theta_1, \theta_2 \in \RR^d.
\#
Meanwhile, by the definition of meta-objective in \eqref{eq::def_sl_meta_obj}, we have
\#\label{eq::pf_sl_eq001}
L(\omega) - L(\theta^*) = \frac{1}{n}\cdot\sum^n_{i = 1} R_i(h_{\omega_i}) - R_i(h_{\theta^*_i}).
\#
Recall that $\omega_i$ and $\theta^*_i$ are defined as follow,
\$
\omega_i = \omega  - \eta \cdot \nabla_\omega R_i(h_{\omega}), \quad \theta^*_i = \theta^*  - \eta \cdot \nabla_{\theta^*} R_i(h_{\theta^*}), \quad \forall i \in [n].
\$ 
By plugging \eqref{eq::pf_sl_eq000} into \eqref{eq::pf_sl_eq001} with $\theta_2 = \omega_i$ and $\theta_1 = \theta^*_i$, respectively, for $i\in[n]$, we have
\#\label{eq::pf_sl_eq002}
L(\omega) - L(\theta^*) \leq \frac{1}{n}\cdot \sum^n_{i = 1}\langle \delta R_i/\delta h_{\omega_i},  h_{\omega_i} - h_{\theta^*_i} \rangle_\cH.
\#
Thus, combining \eqref{eq::pf_sl_eq002} and the definition of the $\epsilon$-stationary point $\omega$ in \eqref{eq::eps_optimal}, we have
\#\label{eq::pf_sl_eq003}
L(\omega) - L(\theta^*) \leq R\cdot \epsilon -\nabla_\omega L(\omega)^\top (R\cdot v) + \frac{1}{n}\cdot \sum^n_{i = 1}\langle \delta R_i/\delta h_{\omega_i},  h_{\omega_i} - h_{\theta^*_i} \rangle_\cH, 
\#
which holds for all $R >0$ and $v \in \cB = \{\theta \in \RR^d: \|\theta\|_2 \leq 1\}$. 

It suffices to upper bound the right-hand side of \eqref{eq::pf_sl_eq003}. To this end, we first compute the gradient $\nabla_\omega L(\omega)$. By the chain rule, we obtain for all $R_i$ defined in \eqref{eq::sl_task} and $v \in \RR^d$ that
\#\label{eq::pf_sl_eq1}
\nabla_\omega R_i(h_{\omega_i})^\top v = \langle \delta R_i/\delta h_{\omega_i}, (\ud h_{\omega_i}/\ud \omega)^\top v \rangle_\cH,\quad \forall i \in [n].
\#
Meanwhile, by the definition of $\omega_i$ in \eqref{eq::pf_sl_eq2} and the parameterization of hypothesis defined in \eqref{eq::def_H_theta}, we obtain from the chain rule that
\#\label{eq::pf_sl_eq3}
(\ud h_{\omega_i}/\ud \omega)(\cdot) = \bigl(I_d - \eta\cdot \nabla^2_\omega R_i(h_{\omega}) \bigr) \phi(\cdot),
\#
where $I_d$ is the identity matrix of size $d\times d$. By the Leibniz integral rule, we have
\#\label{eq::pf_sl_eq4}
\nabla^2_\omega R_i(h_{\omega}) &= \nabla^2_\omega\EE_{(x, y)\sim \cD_i} \Bigl[\ell\bigl( \phi(x)^\top\omega, (x, y) \bigr)\Bigr]= \int_{\cX\times\cY} \nabla^2_\omega \ell\bigl( \phi(x)^\top \omega, (x, y) \bigr) \ud \cD_i(x, y).
\#
In what follows, we write 
\#\label{eq::pf_sl_N}
N (\omega, x, y) = I_d - \eta\cdot\nabla^2_{\omega}\ell\bigl(\phi(x)^\top\omega, (x, y)\bigr)
\# 
for notational simplicity. By plugging \eqref{eq::pf_sl_eq3} and \eqref{eq::pf_sl_eq4} into \eqref{eq::pf_sl_eq1}, we obtain that
\#\label{eq::pf_sl_eq5}
\nabla_\omega R_i(h_{\omega_i})^\top v = \int_{\cX\times\cY\times\cX}  (\delta R_i/\delta h_{\omega_i})(x')\cdot \bigl(N(\omega, x, y) \phi(x')\bigr)^\top v\ud \cD_i(x, y) \ud \rho(x').
\#
Thus, by the definition of meta-objective in \eqref{eq::def_sl_meta_obj}, we have for all $v \in\RR^d$ that
\#\label{eq::pf_sl_eq6}
\nabla_\omega L(\omega)^\top v &= \frac{1}{n}\cdot \sum^n_{i = 1}\nabla_\omega R_i(h_{\omega_i})^\top v \notag\\
&= \frac{1}{n} \cdot\sum^n_{i = 1}\int_{\cX\times\cY\times\cX}  (\delta R_i/\delta h_{\omega_i})(x')\cdot \bigl(N(\omega, x, y) \phi(x')\bigr)^\top v\ud \cD_i(x, y) \ud \rho(x').
\#

By plugging \eqref{eq::pf_sl_eq6} into \eqref{eq::pf_sl_eq003}, we have
\#\label{eq::pf_sl_eq7}
&L(\omega) - L(\theta^*) \notag\\
&\quad\leq R\cdot \epsilon + \int_\cX \frac{1}{n}\cdot \sum^n_{i = 1} (\delta R_i/\delta h_{\omega_i})(x')\cdot  \bigl(h_{\omega_i}(x') - h_{\theta^*_i}(x') \bigr) \ud \rho(x')\notag\\
&\quad\qquad - \frac{1}{n}\cdot  \sum^n_{i = 1}\int_{\cX\times\cY\times\cX}  (\delta R_i/\delta h_{\omega_i})(x')\cdot \bigl(N(\omega, x, y) \phi(x')\bigr)^\top (R\cdot v)\ud \cD_i(x, y) \ud \rho(x'),
\#
which holds for all $v\in\cB$. Meanwhile, it holds for all $i\in[n]$ that 
\#\label{eq::pf_sl_eq8}
&\int_{\cX\times\cY\times\cX}  (\delta R_i/\delta h_{\omega_i})(x')\cdot \bigl(N(\omega, x, y) \phi(x')\bigr)^\top (R\cdot v)\ud \cD_i(x, y) \ud \rho(x')\\
&\quad = \int_{\cX\times\cY\times\cX}  (\delta R_i/\delta h_{\omega_i})(x')\cdot (\ud \cD_i/\ud \cM)(x, y) \cdot \bigl(N(\omega, x, y) \phi(x')\bigr)^\top (R\cdot v)\ud \cM(x, y)\ud \rho(x'),\notag
\#
where $\cM$ is the mixed distribution defined in \eqref{eq::def_mix_sl} and $\ud \cD_i/\ud \cM$ is the Radon-Nikodym derivative. Thus, by plugging \eqref{eq::pf_sl_eq8} into \eqref{eq::pf_sl_eq7}, we obtain for all $R >0$ and $v\in\cB$ that
\#\label{eq::pf_sl_eq9}
&L(\omega) - L(\theta^*) \notag\\
&\quad\leq R\cdot \epsilon +\int_\cX \frac{1}{n}\cdot \sum^n_{i = 1} (\delta R_i/\delta h_{\omega_i})(x')\cdot  \bigl(h_{\omega_i}(x') - h_{\theta^*_i}(x') \bigr) \ud \rho(x')\\
&\quad\qquad - \int_{\cX\times\cY\times\cX}  \frac{1}{n}\cdot \sum^n_{i = 1}  (\delta R_i/\delta h_{\omega_i})(x')\cdot  (\ud \cD_i/\ud \cM)(x, y)\cdot \bigl(N(\omega, x, y) \phi(x')\bigr)^\top (R\cdot v)\ud \cM(x, y)\ud \rho(x')\notag\\
&\quad \leq R\cdot \epsilon +\| w \|_{\cM\cdot\rho}\cdot \biggl(\int _{\cX\times\cY\times\cX} u(x, y, x')  - \bigl(N(\omega, x, y)\phi(x')\bigr)^\top (R\cdot v)\ud \cM(x, y)\ud \rho(x')\biggr)^{1/2},\notag
\#
where the second inequality follows from the Cauchy-Schwartz inequality, and $N$ is defined in \eqref{eq::pf_sl_N}. Here we define 
\$
w(x, y, x') &= \frac{1}{n}\cdot \sum^n_{i = 1}  (\delta R_i/\delta h_{\omega_i})(x')\cdot  (\ud \cD_i/\ud \cM)(x, y),\notag\\
u(x,y, x') &= \biggl(\frac{1}{n}\cdot \sum^n_{i = 1} (\delta R_i/\delta h_{\omega_i})(x')\cdot  \bigl(h_{\omega_i}(x') - h_{\theta^*_i}(x') \bigr)\biggr) \bigg/w(x, y, x'),
\$
and we define $\|w\|_{\cM\cdot \rho} = (\int w^2(x, y, x') \ud \cM(x, y)\ud \rho(x'))^{1/2}$ the $L_2(\cM\cdot\rho)$-norm of $w$. Thus, by taking the infimum on the right-hand side of \eqref{eq::pf_sl_eq9} over $v\in\cB$ and setting $\phi_{\ell, \omega}(x, y, x') = N(\omega, x, y)\phi(x')$ for all $(x, y, x')\in\cX\times\cY\times\cX$, we complete the proof of Theorem \ref{thm::sl_opt}.
\end{proof}

\subsection{Proof of Corollary \ref{cor::opt_sl_sqloss}}
\label{pf::opt_sl_sqloss}
\begin{proof}
The proof is similar to that of Theorem \ref{thm::sl_opt} in \S\ref{pf::sl_opt}. By \eqref{eq::pf_sl_eq003} in the proof of Theorem \ref{thm::sl_opt}, we have
\#\label{eq::pf_sqloss_eq1}
L(\omega) - L(\theta^*) \leq R\cdot \epsilon -\nabla_\omega L(\omega)^\top (R\cdot v) + \frac{1}{n}\cdot \sum^n_{i = 1}\langle \delta R_i/\delta h_{\omega_i},  h_{\omega_i} - h_{\theta^*_i} \rangle_\cH.
\#
In what follows, we upper bound the right-hand side of \eqref{eq::pf_sqloss_eq1}. To this end, we first compute the gradient $\nabla_\omega L(\omega)$. By \eqref{eq::pf_sl_eq5} in the proof of Theorem \ref{thm::sl_opt} in \S\ref{pf::sl_opt}, we have
\#\label{eq::pf_sqloss_eq2}
\nabla_\omega R_i(h_{\omega_i})^\top v = \int_{\cX\times\cY\times\cX}  (\delta R_i/\delta h_{\omega_i})(x')\cdot \bigl(N(\omega, x, y) \phi(x')\bigr)^\top v\ud \cD_i(x, y) \ud \rho(x'),
\#
where
\#\label{eq::pf_sqloss_eqN}
N (\omega, x, y) = I_d - \eta\cdot\nabla^2_{\omega}\ell\bigl(\phi(x)^\top\omega, (x, y)\bigr).
\#
Note that for $\ell(h, (x, y)) = (y - h(x))^2$, we have
\#\label{eq::pf_sqloss_eqN1}
\nabla^2_{\omega}\ell\bigl(\phi(x)^\top\omega, (x, y)\bigr) = 2 \phi(x)\phi(x)^\top,
\#
which does not depend on $y$. Thus, for $\overline \cD_i(x)= \int_\cY \cD_i(x, y)\ud y = \rho(x)$, we obtain from \eqref{eq::pf_sqloss_eqN} and \eqref{eq::pf_sqloss_eqN1} that
\#\label{eq::pf_sqloss_eqN2}
\int_{\cX\times\cY} N(\omega, x, y) \ud \cD_i(x, y) = I_d - \eta\cdot \int_\cX 2\phi(x)\phi(x)^\top \ud \rho(x) =\EE_{x\sim \rho} \bigl[I_d - 2\eta\cdot\phi(x)\phi(x)^\top \bigr].
\#
By further plugging \eqref{eq::pf_sqloss_eqN2} into \eqref{eq::pf_sqloss_eq2}, we obtain that
\#\label{eq::pf_sqloss_eq3}
\nabla_\omega R_i(h_{\omega_i})^\top v = \int_{\cX}  (\delta R_i/\delta h_{\omega_i})(x')\cdot \bigl(K_\eta \phi(x')\bigr)^\top v\ud \rho(x'),
\#
where we define 
\#\label{eq::pf_sqloss_K}
K_\eta = \EE_{x\sim \rho} \bigl[I_d - 2\eta\cdot\phi(x)\phi(x)^\top \bigr].
\#
Thus, by the definition of meta-objective in \eqref{eq::def_sl_meta_obj}, it holds for all $v \in\RR^d$ that
\#\label{eq::pf_sqloss_eq4}
\nabla_\theta L(\omega)^\top v &= \frac{1}{n}\cdot \sum^n_{i = 1}\nabla_\omega R_i(h_{\omega_i})^\top v = \frac{1}{n}\cdot \sum^n_{i = 1}\int_{\cX}  (\delta R_i/\delta h_{\omega_i})(x')\cdot \bigl(K_\eta \phi(x')\bigr)^\top v \ud \rho(x'),
\#
where $K_\eta$ is defined in \eqref{eq::pf_sqloss_K}. By plugging \eqref{eq::pf_sqloss_eq4} into \eqref{eq::pf_sqloss_eq1}, we have
\#\label{eq::pf_sqloss_eq5}
&L(\omega) - L(\theta^*) \notag\\
&\quad\leq R\cdot \epsilon +\int_\cX \frac{1}{n}\cdot \sum^n_{i = 1} (\delta R_i/\delta h_{\omega_i})(x')\cdot  \bigl(h_{\omega_i}(x') - h_{\theta^*_i}(x') \bigr) \ud \rho(x')\notag\\
&\quad\qquad - \int_\cX \frac{1}{n}\cdot\sum^n_{i = 1}(\delta R_i/\delta h_{\omega_i})(x')\cdot \bigl(K_\eta \phi(x')\bigr)^\top(R\cdot v) \ud \rho(x'),
\#
which holds for all $v \in \cB = \{\theta \in \RR^d: \|\theta\|_2 = 1\}$. By the Cauchy-Schwartz inequality, we obtain from \eqref{eq::pf_sqloss_eq5} that
\#\label{eq::pf_sqloss_eq6}
L(\omega) - L(\theta^*) \leq R\cdot \epsilon + \|w\|_\rho\cdot \|u -(K_\eta\phi)^\top (R\cdot v)\|_\rho,
\# 
which holds for all $v \in \cB$. Here we define for all $x'\in\cX$ that
\#\label{eq::pf_sqloss_eq7}
w(x') &= \frac{1}{n}\cdot\sum^n_{i = 1}(\delta R_i/\delta h_{\omega_i})(x'), \notag\\
u(x')& = \biggl( \frac{1}{n}\cdot \sum^n_{i = 1} (\delta R_i/\delta h_{\omega_i})(x')\cdot  \bigl(h_{\omega_i}(x') - h_{\theta^*_i}(x') \bigr)\biggr)\bigg/w(x').
\#

It remains to upper bound the norm $\|w\|_\rho$, where $w$ is defined in \eqref{eq::pf_sqloss_eq7}. By Proposition \ref{prop::sqloss_diff}, it holds for all $x'\in\cX$ that
\#\label{eq::pf_sqloss_diff}
(\delta R_i/\delta h)(x')  = 2\EE_{(x, y)\sim \cD_i}\bigl[h(x)- y~\big | ~ x = x'\bigr].
\#
Thus, by the fact that $\overline \cD_i(x) = \int_\cY \cD_i(x, y) \ud y = \rho(x)$, we obtain from \eqref{eq::pf_sqloss_diff} that
\#\label{eq::pf_sqloss_eq8}
\|\delta R_i/\delta h\|^2_\rho &= 4\int \Bigl\{\EE_{(x, y)\sim \cD_i}\bigl[h(x)- y~\big | ~ x = x'\bigr]\Bigr\}^2 \ud \rho(x') \notag\\
&\leq 4 \int \bigl(h(x) - y\bigr)^2 \ud \cD_i(y\given x) \ud \rho(x)\notag\\
&= 4 \EE_{(x, y)\sim \cD_i}\Bigl[\bigl(y - h(x)\bigr)^2\Bigr],
\#
where the second inequality follows from Jensen's inequality, and we denote by $\cD_i(y\given x)$ the conditional distribution of $y$ given $x$ for $(x, y)\sim \cD_i$. Thus, following from \eqref{eq::pf_sqloss_eq8} and the definition of $w$ in \eqref{eq::pf_sqloss_eq7}, we obtain that
\#\label{eq::pf_sqloss_eq9}
\|w\|_\rho \leq \frac{1}{n} \sum^n_{i = 1} \|\delta R_i/\delta h_{\omega_i}\|_\rho \leq \frac{2}{n}\cdot \sum^n_{i = 1} \Bigl\{\EE_{(x, y)\sim \cD_i}\Bigl[\bigl(y - h_{\omega_i}(x)\bigr)^2\Bigr]\Bigr\}^{1/2}.
\#
Finally, by plugging \eqref{eq::pf_sqloss_eq9} into \eqref{eq::pf_sqloss_eq7}, we have
\#\label{eq::pf_sqloss_eq10}
L(\omega) - L(\theta^*) \leq R\cdot \epsilon + 2 \overline R\cdot \|u - (K_\eta\phi)^\top(R\cdot v)\|_\rho,
\#
which holds for all $v \in \cB$. Here we define $u$ and $K_\eta$ in \eqref{eq::pf_sqloss_eq7} and \eqref{eq::pf_sqloss_K}, respectively, and we define $\overline R$ as follows,
\$
\overline R  = \frac{1}{n}\sum^n_{i = 1} R_i^{1/2}(h_{\omega_i})=\frac{1}{n}\sum^n_{i = 1}\Bigl\{\EE_{(x, y) \sim \cD_i}\Bigl[\bigl(y - h_{\omega_i}(x)\bigr)^2\Bigr]\Bigr\}^{1/2}.
\$
Thus, by taking the infimum over $v \in \cB = \{\theta\in\RR^d: \|\theta\|_2 \leq 1\}$ on the right-hand side of \eqref{eq::pf_sqloss_eq10}, we complete the proof of Corollary \ref{cor::opt_sl_sqloss}.
\end{proof}

\subsection{Proof of Corollary \ref{cor::eps_nn_opt}}
\label{pf::eps_nn_opt}
\begin{proof}
The proof hinges on the following lemma, which is adapted from \cite{cai2019neural}.
\begin{lemma}[Linearization Error \citep{cai2019neural}]
\label{lem::lin_err}
Under Assumption \ref{asu::reg_cond_rl}, it holds for $\omega_0, \omega_1, \omega_2 \in \cB = \{\theta\in\RR^{md}:\|\theta - W_{\text{\rm init}}\|_2 \leq R\}$ that
\$
\EE_{\text{\rm init}}\bigl[\|\phi_{\omega_0}(\cdot,\cdot)^\top \omega_2 - \phi_{\omega_1}(\cdot, \cdot)^\top\omega_2\|^2_{\varrho_{\pi_\theta}}\bigr] = \cO(R^{3}\cdot m^{-1/2}),
\$
where $\varrho_{\pi_\theta}$ is the mixed visitation measure defined in \eqref{eq::def_mixed_meta} of Definition \ref{def::meta_meta_visit}.
\end{lemma}
\begin{proof}
See \S\ref{pf::lin_err} for a detailed proof.
\end{proof}

Note that $\nabla_\omega f((s, a); \omega) = \phi_{\omega}(s, a)$, which holds almost everywhere for $(s, a)\in\cS\times\cA$. Here $\phi_\omega$ is the feature mapping defined in \eqref{eq::def_nn_phi} with $W = \omega$. Hence, following from similar analysis to that in the proof of Theorem \ref{thm::eps_opt} in \S\ref{pf::thm_eps_opt}, we obtain that
\#\label{eq::pf_nnrl_eq1}
L(\theta^*) - L(\omega) &\leq \epsilon + C\cdot \| f_{\omega}(\cdot, \cdot) - \phi_\omega(\cdot, \cdot)^\top (v- \omega)\|_{\varrho_{\pi_\omega}},
\#
which holds for all $v \in \cB_{\text{\rm init}}$. Here $C = 2C_0\cdot Q_{\max}/\tau\cdot(1 + 2Q_{\max}\cdot\overline\gamma\cdot\eta)$, $C_0$ is defined in Assumption \ref{asu::concen_coeff}, and $\overline \gamma = (\sum^n_{i = 1}\gamma_i)/n$. Meanwhile, we define
\#\label{eq::pf_nnrl_f}
f_\omega(s', a') =   \biggl(\sum^n_{i = 1}\frac{A^{\pi_{i, \omega}}_i(s', a')}{1 - \gamma_i}\cdot \frac{\ud \sigma_{\pi_{i, \theta^*}}}{ \ud \varrho_{\pi_\omega}}(s', a')\biggr)\bigg/\biggl(\sum^n_{i = 1} g_{i, \omega}(s', a') \cdot \frac{\ud \varsigma_{i, \pi_{\omega}}}{ \ud \varrho_{\pi_\omega}}(s', a')\biggr),
\#
where $g_{i, \omega}$ is defined in \eqref{eq::def_g} of Lemma \ref{lem::meta_grad_refine}. In what follows, we define $v_0 \in \cB_{\text{\rm init}}$ as follows,
\#\label{eq::pf_nnrl_v0}
v_0 \in \argmin_{v\in\cB_{\text{\rm init}}}\|f_{\omega}(\cdot, \cdot)+ \phi_{\omega}(\cdot, \cdot)^\top \omega - \phi_v(\cdot, \cdot)^\top v\|_{\varrho_{\pi_\omega}}.
\#
It then holds from \eqref{eq::pf_nnrl_eq1} that
\#\label{eq::pf_nnrl_eq2}
L(\theta^*) - L(\omega) &\leq \epsilon + C\cdot \bigl\|f_{\omega}(\cdot, \cdot) - \phi_\omega(\cdot, \cdot)^\top (v_0 - \omega)\|_{\varrho_{\pi_\omega}}\\
&\leq \epsilon + C\cdot \|f_{\omega}(\cdot, \cdot) + \phi_{\omega}(\cdot, \cdot)^\top \omega - \phi_{v_0}(\cdot, \cdot)^\top v_0\|_{\varrho_{\pi_\omega}} + \bigl\|\bigl(\phi_{\omega}(\cdot, \cdot) - \phi_{v_0}(\cdot, \cdot)\bigr)^\top v_0\bigr\|_{\varrho_{\pi_\omega}}.\notag
\#
Note that $v_0, \omega \in \cB_{\text{\rm init}}$. Thus, following from Lemma \ref{lem::lin_err}, upon taking expectation of \eqref{eq::pf_nnrl_eq2} over the random initialization, we obtain that
\#\label{eq::pf_nnrl_eq3}
\EE_{\text{\rm init}}\bigl[L(\theta^*) - L(\omega)\bigr] \leq  \epsilon + C\cdot \EE_{\text{\rm init}}\bigl[\|f_{\omega}(\cdot, \cdot) + \phi_{\omega}^\top \omega - \phi_{v_0}(\cdot, \cdot)^\top v_0\|_{\varrho_{\pi_\omega}}\bigr] + \cO(R_T^{3/2}\cdot m^{-1/4}).
\#
Note that by the definition of neural network and feature mapping in \eqref{eq::def_nn} and \eqref{eq::def_nn_phi}, respectively, we have 
\$
f((\cdot, \cdot); v_0) = \phi_{v_0}(\cdot, \cdot)^\top v_0, \quad f((\cdot, \cdot); \omega) = \phi_{\omega}(\cdot, \cdot)^\top \omega.
\$
Thus, by plugging the definition of $v_0$ in \eqref{eq::pf_nnrl_v0} into \eqref{eq::pf_nnrl_eq3} and setting $c_\omega(\cdot, \cdot) = f((\cdot, \cdot); \omega)+ f_\omega(\cdot, \cdot)$, we have
\$
\EE_{\text{\rm init}}\bigl[L(\theta^*) - L(\omega)\bigr] &\leq  \epsilon + C\cdot \EE_{\text{\rm init}}\Bigl[ \inf_{v\in\cB} \bigl\|c_{\omega}(\cdot, \cdot) - f\bigl((\cdot, \cdot); v\bigr)\bigr\|_{\varrho_{\pi_\omega}}\Bigr] + \cO(R_T^{3/2}\cdot m^{-1/4}),
\$
which completes the proof of Corollary \ref{cor::eps_nn_opt}.
\end{proof}

\subsection{Proof of Corollary \ref{cor::nnsl}}
\label{pf::nnsl}
\begin{proof}
The proof is similar to that of Corollary \ref{cor::opt_sl_sqloss} in \S\ref{pf::opt_sl_sqloss}. Similar to \eqref{eq::pf_sl_eq003} in the proof of Theorem \ref{thm::sl_opt}, we have
\#\label{eq::pf_nnsl_eq1}
L(\omega) - L(\theta^*) \leq  \epsilon -\nabla_\omega L(\omega)^\top  (\omega - v) + \frac{1}{n}\cdot \sum^n_{i = 1}\langle \delta R_i/\delta h_{\omega_i},  h_{\omega_i} - h_{\theta^*_i} \rangle_\cH,
\#
which holds for all $v \in\cB_{\text{\rm init}}$. It then suffices to upper bound the right-hand side of \eqref{eq::pf_nnsl_eq1}. To this end, we first calculate the gradient $\nabla_\omega L(\omega)$. Similar to \S\ref{pf::opt_sl_sqloss}, we have
\#\label{eq::pf_nnsl_eq2}
&\nabla_\omega R_i(h_{\omega_i})^\top (\omega-v) \notag\\
&\qquad= \int_{\cX\times\cY\times\cX}  (\delta R_i/\delta h_{\omega_i})(x')\cdot \bigl(N(\omega, x, y) \phi_{\omega_i}(x')\bigr)^\top(\omega- v)\ud \cD_i(x, y) \ud \rho(x'),
\#
where $\phi_{\omega_i}$ is the feature mapping defined in \eqref{eq::def_nn_phi} with $W = \omega_i$ and
\#\label{eq::pf_nnsl_eqN}
N (\omega, x, y) = I_{md} - \eta\cdot\nabla^2_{\omega}\ell\bigl(h_{\omega}, (x, y)\bigr).
\#
Note that for $\ell(h, (x, y)) = (y - h(x))^2$, we have
\#\label{eq::pf_nnsl_eqN1}
\nabla^2_{\omega}\ell\bigl(h_{\omega}, (x, y)\bigr) = 2 \bigl(\nabla_\omega h_\omega(x)\bigr)\bigl(\nabla_\omega h_\omega(x)\bigr)^\top + 2(h_{\omega}(x) - y)\nabla^2_\omega h_\omega(x).
\#
Meanwhile, by the parameterization $h_\omega(x) = f(x; \omega)$ defined in \eqref{eq::def_nn},we obtain that $\nabla_\omega h_\omega(x) = \phi_\omega(x)$ and $\nabla^2_\omega h_\omega(x) = 0$, which holds almost everywhere for $x\in \cX$. Thus, it follows from \eqref{eq::pf_sqloss_eqN1} that
\#\label{eq::pf_nnsl_eqN2}
\nabla^2_{\omega}\ell\bigl(h_{\omega}, (x, y)\bigr) = 2\phi_\omega(x)\phi_\omega(x)^\top, 
\#
which holds almost everywhere for $x\in\cX$. Moreover, for a fixed $x$, \eqref{eq::pf_nnsl_eqN2} holds almost everywhere on $\omega\in\RR^{md}$. Here recall that $\phi_\omega$ is the feature mapping defined in \eqref{eq::def_nn_phi} with $W = \omega$. 
Hence, we can obtain \eqref{eq::pf_nnsl_eqN2} uniformly for all $x\in\cX$ by setting the second order derivative of the neural network with respect to the parameter to be zero in the optimization of meta-objective when it is infinite (which occurs with zero probability and does not affect the convergence of meta-SL). By plugging \eqref{eq::pf_nnsl_eqN2} and \eqref{eq::pf_nnsl_eqN} into \eqref{eq::pf_nnsl_eq2}, we have
\#\label{eq::pf_nnsl_eq3}
\nabla_\omega R_i(h_{\omega_i})^\top v = \int_{\cX}  (\delta R_i/\delta h_{\omega_i})(x')\cdot  \phi_{\omega_i}(x')^\top K_{\omega, \eta} (\omega-v) \ud \rho(x'),
\#
where we define
\#\label{eq::pf_nnsl_K}
K_{\omega, \eta} =I_{md} - \eta\cdot \EE_{x\sim \rho}\bigl[2\phi_\omega(x)\phi_\omega(x)^\top\bigr].
\#
In the sequel, we define 
\#\label{eq::pf_nnsl_phi_0}
\phi_0(x) = \phi_{W_{\text{\rm init}}}(x)
\#
for notational simplicity, where $W_{\text{\rm init}}$ is the initial parameter of the neural network, and $\phi_{W_{\text{\rm init}}}$ is the feature mapping defined in \eqref{eq::def_nn_phi} with $W = W_{\text{\rm init}}$. It then follows from \eqref{eq::pf_nnsl_eq3} that
\#\label{eq::pf_nnsl_eq4}
\nabla_\omega R_i(h_{\omega_i})^\top (\omega - v) &= \int_{\cX}  (\delta R_i/\delta h_{\omega_i})(x')\cdot  \phi_{0}(x')^\top K_{\omega, \eta} (\omega-v) \ud \rho(x') \\
&\qquad+ \int_{\cX}  (\delta R_i/\delta h_{\omega_i})(x')\cdot \bigl( \phi_{\omega_i} - \phi_{0}(x')\bigr)^\top K_{\omega, \eta} (\omega-v)\ud \rho(x').\notag
\#
Thus, by the definition of meta-objective in \eqref{eq::def_sl_meta_obj}, it follows from \eqref{eq::pf_nnsl_eq4} that
\#\label{eq::pf_nnsl_eq5}
\nabla_\omega L(\omega)^\top  (\omega - v) &= \frac{1}{n}\cdot\sum^n_{i = 1}\nabla_\omega R_i(h_{\omega_i})^\top (\omega - v)\notag\\
&= \frac{1}{n}\cdot\sum^n_{i = 1}\int_{\cX}  (\delta R_i/\delta h_{\omega_i})(x')\cdot  \phi_{0}(x')^\top K_{\omega, \eta} (\omega-v) \ud \rho(x')+ \frac{1}{n}\cdot\sum^n_{i = 1} P_i,
\#
where recall that we define $\phi_0$ in \eqref{eq::pf_nnsl_phi_0}, and we define
\#\label{eq::pf_nnsl_Pi}
P_i = \int_{\cX}  (\delta R_i/\delta h_{\omega_i})(x')\cdot \bigl( \phi_{\omega_i} - \phi_{0}(x')\bigr)^\top K_{\omega, \eta} (\omega-v)\ud \rho(x').
\#

Similar to \eqref{eq::pf_sqloss_eq6} in the proof of Corollary \ref{cor::opt_sl_sqloss} in \S\ref{pf::opt_sl_sqloss}, we obtain from \eqref{eq::pf_nnsl_eq1} and \eqref{eq::pf_nnsl_eq5} that
\#\label{eq::pf_nnsl_eq7}
L(\omega) - L(\theta^*) \leq  \epsilon +  \int_{\cX}w(x')\cdot \bigl(u(x') - \phi_{0}(x')^\top K_{\omega, \eta} (\omega-v)\bigr) \ud \rho(x') +\frac{1}{n}\cdot \sum^n_{i=1}P_i,
\#
which holds for all $v\in\cB_{\text{\rm init}}$. Here we define for all $x'\in\cX$ that
\#\label{eq::pf_nnsl_eq8}
w(x') &=\frac{1}{n}\cdot\sum^n_{i = 1}(\delta R_i/\delta h_{\omega_i})(x')\notag\\
u(x') &= \biggl( \sum^n_{i = 1} (\delta R_i/\delta h_{\omega_i})(x')\cdot  \bigl(h_{\omega_i}(x') - h_{\theta^*_i}(x') \bigr)\biggr) \bigg/\biggl(\sum^n_{i = 1}  (\delta R_i/\delta h_{\omega_i})(x')\biggr).
\#
In what follows, we fix $v \in \cB_{\text{\rm init}}$ as follows,
\#\label{eq::pf_nnsl_eq9}
v \in \argmin_{\theta \in \cB_{\text{\rm init}}} \bigl\|u(\cdot) - f\bigl(\cdot; K_{\omega, \eta}(\omega - \theta) + W_{\text{\rm init}}\bigr)\bigr\|_\rho.
\#
Meanwhile, we define
\#\label{eq::pf_nnsl_s}
s = K_{\omega, \eta}(\omega - v) + W_{\text{\rm init}},
\#
where $v$ is fixed in \eqref{eq::pf_nnsl_eq9}. Note that $f(x; W_{\text{\rm init}})=\phi_0(x)^\top W_{\text{\rm init}} = 0$ for all $x\in\cX$ by the initialization of neural networks. It then holds from \eqref{eq::pf_nnsl_eq7} that
\#\label{eq::pf_nnsl_eq10}
L(\omega) - L(\theta^*) &\leq  \epsilon +  \int_{\cX}w(x')\cdot \Bigl(u(x') - \phi_{0}(x')^\top \bigl(K_{\omega, \eta} (\omega-v) + W_{\text{\rm init}}\bigr)\Bigr) \ud \rho(x') + \frac{1}{n}\cdot\sum^n_{i=1}P_i\notag\\
&\leq  \epsilon + \int_{\cX}w(x')\cdot \bigl(u(x') - \phi_{s}(x')^\top s \bigr)\ud \rho(x')+P_0 + \frac{1}{n}\cdot\sum^n_{i=1}P_i,
\#
where we define $P_i$ and $s$ in \eqref{eq::pf_nnsl_Pi} and \eqref{eq::pf_nnsl_s}, respectively, and we define
\#\label{eq::pf_nnsl_P0}
P_0 = \int_{\cX}w(x')\cdot \bigl(\phi_0(x') - \phi_{s}(x')\bigr)^\top s \ud \rho(x').
\#
By further plugging \eqref{eq::pf_nnsl_eq9} into \eqref{eq::pf_nnsl_eq10}, we have
\#\label{eq::pf_nnsl_eq11}
L(\omega) - L(\theta^*) &\leq \epsilon + \|w\|_\rho\cdot \inf_{\theta \in \cB_0} \|u(\cdot) - f(\cdot; \theta)\|_\rho + P_0 + \frac{1}{n}\cdot\sum^n_{i=1}P_i,
\#
where we define
\#\label{eq::pf_nnsl_B0}
\cB_0 = K_{\omega, \eta}(\omega - \cB_{\text{\rm init}}) + W_{\text{\rm init}}.
\#
Following from \eqref{eq::pf_sqloss_eq9} in the proof of Corollary \ref{cor::opt_sl_sqloss} in \S\ref{pf::opt_sl_sqloss}, we obtain that
\#\label{eq::pf_nnsl_w}
\|w\|_\rho \leq 2\overline R = \frac{2}{n}\sum^n_{i = 1} R_i^{1/2}(h_{\omega_i})=\frac{1}{n}\sum^n_{i = 1}\Bigl\{\EE_{(x, y) \sim \cD_i}\Bigl[\bigl(y - h_{\omega_i}(x)\bigr)^2\Bigr]\Bigr\}^{1/2}.
\#

\vskip4pt
\noindent{\bf Upper Bounding $P_0$ and $P_i$ for $i\in[n]$.} It remains to upper bound the terms $P_0$ and $P_i$ in \eqref{eq::pf_nnsl_eq11} for $i\in[n]$. By the Cauchy-Schwartz inequality and the definition of $P_0$ in \eqref{eq::pf_nnsl_P0}, we have
\#\label{eq::pf_nnsl_eq12}
\EE_{\text{\rm init}}[P_0] \leq \Bigl\{\EE_{\text{\rm init}}\bigl[\|w\|^2_\rho\bigr]\cdot \EE_{\text{\rm init}}\bigl[\|\phi_{s}(\cdot)^\top s - \phi_0(\cdot)^\top s \|^2_\rho\bigr]\Bigr\}^{1/2}.
\#
Meanwhile, for $s$ and $K_{\omega, \eta}$ defined in \eqref{eq::pf_nnsl_s} and \eqref{eq::pf_nnsl_K}, respectively, we have
\#\label{eq::pf_nnsl_eq13}
\|s - W_{\text{\rm init}}\|_2 &\leq \|\omega - v\|_2 + 2\eta\cdot \Bigl\|\EE_{x\sim \rho}\bigl[\phi_\omega(x)\phi_\omega(x)^\top(\omega - v)\bigr] \Bigr\|_2\notag\\
&\leq \|\omega - v\|_2 + 2\eta\cdot \EE_{x\sim\rho}\bigl[\|\phi_\omega(x)\phi_\omega(x)^\top(\omega - v)\|_2\bigr]\notag\\
&\leq (1 + 2\eta)\cdot\|\omega - v\|_2 \leq (2 + 4\eta)\cdot R_T,
\#
where the first inequality follows from the triangle inequality, the second inequality follows from Jensen's inequality, the third inequality follows from the fact that $\|\phi_\omega(x)\|_2 \leq 1$ for all $x\in\cX$, and the fourth inequality follows from the fact that $\omega, v \in \cB_T$. Hence, by \eqref{eq::pf_nnsl_eq13} and Lemma \ref{lem::lin_err}, we obtain that
\#\label{eq::pf_nnsl_eq14}
\EE_{\text{\rm init}}\bigl[\|\phi_{s}(\cdot)^\top s - \phi_0(\cdot)^\top s \|^2_\rho\bigr] = \cO\bigl((1 + 2\eta)^{3}\cdot R_T^{3}\cdot m^{-1/2}\bigr).
\#
By further plugging \eqref{eq::pf_nnsl_eq14} and \eqref{eq::pf_nnsl_w} into \eqref{eq::pf_nnsl_eq12}, we have
\#\label{eq::pf_nnsl_eq15}
\EE_{\text{\rm init}}[P_0] = \cO\bigl(\overline C\cdot(1 + 2\eta)^{3/2}\cdot R_T^{3/2}\cdot m^{-1/4}\bigr).
\#
Here we define
\#\label{eq::def_C_bar}
\overline C = \{\EE_{\text{\rm init}}[4\overline R^2]\}^{1/2},
\#
where $\overline R$ is defined in \eqref{eq::pf_nnsl_w}. 

Similarly, for $P_i$ defined in \eqref{eq::pf_nnsl_Pi}, we have
\#\label{eq::pf_nnsl_eq16}
\EE_{\text{\rm init}}[P_i] \leq \Bigl\{\EE_{\text{\rm init}}\bigl[\|\delta R_i/\delta h_{\omega_i}\|^2_\rho\bigr]\cdot \underbrace{ \EE_{\text{\rm init}}\bigl[\| \phi_{\omega_i}(\cdot)^\top K_{\omega, \eta} (\omega-v) - \phi_{0}^\top K_{\omega, \eta} (\omega-v) \|^2_\rho\bigr]}_{\textstyle{U}}\Bigr\}^{1/2}.
\#
We first upper bound the term $U$ in \eqref{eq::pf_nnsl_eq16}. By the definition of $\omega_i$ in \eqref{eq::pf_sl_eq2}, we have
\#\label{eq::pf_nnsl_eq17}
\|\omega_i - W_{\text{\rm init}}\|_2 \leq \|\omega - W_{\text{{\rm init}}}\|_2 + \eta \cdot \|\nabla_\omega R_i(h_\omega)\|_2.
\#
Meanwhile, by the definition of risk in \eqref{eq::sl_task} and the definition of squared loss in \eqref{eq::def_sl_sqloss}, we have
\#\label{eq::pf_nnsl_eq18}
\|\nabla_\omega R_i(h_\omega)\|_2 &= \Bigl\|\EE_{(x, y)\sim\cD_i}\Bigl[ 2\bigl(h_\omega(x) - y\bigr) \phi_\omega(x) \Bigr]\Bigr\|_2\leq \EE_{(x, y)\sim\cD_i}\bigl[ 2| h_\omega(x) - y | \cdot \|\phi_\omega(x)\|_2 \bigr],
\#
where the first equality follows from the neural network parameterization of $h_\omega$ and the definition of feature mapping in \eqref{eq::def_nn_phi}, and the second inequality follows from Jensen's inequality. Following from the fact that $\|\phi_\omega(x)\|_2\leq 1$ for all $x\in \cX$ and the assumption that $|y| \leq Y_{\max}$ for all $y\in\cY$, we further have
\#\label{eq::pf_nnsl_eq19}
\EE_{(x, y)\sim\cD_i}\bigl[ 2| h_\omega(x) - y | \cdot \|\phi_\omega(x)\|_2 \bigr] \leq 2Y_{\max} + 2\EE_{(x, y)\sim\cD_i}\bigl[ | h_\omega(x) | \bigr].
\#
Note that $f(x; W_{\text{\rm init}}) = \phi_0(x)^\top W_{\text{\rm init}} = 0$ by the initialization. Hence, we have for all $x\in\cX$ that
\#\label{eq::pf_nnsl_eq20}
| h_\omega(x) | &= |f(x;\omega) - f(x; W_{\text{\rm init}})| \notag\\
&\leq \sup_{\theta \in\RR^{md}} \| \nabla_\theta f(x; \theta)\|_2 \cdot \|\omega - W_{\text{\rm init}}\|_2\notag\\
&\leq \|\omega - W_{\text{\rm init}}\|_2 \leq R_T,
\#
where the first equality follows from the neural network parameterization of the hypothesis $h_\omega$, the second inequality follows from the fact that $\| \nabla_\theta f(x; \theta)\|_2 = \|\phi_\theta(x)\|_2 \leq 1$ for all $x \in \cX$, and the last inequality follows from the fact that $\omega \in \cB_T$. By plugging \eqref{eq::pf_nnsl_eq20} into \eqref{eq::pf_nnsl_eq19}, we have
\#\label{eq::pf_nnsl_eq21}
\EE_{(x, y)\sim\cD_i}\bigl[ 2| h_\omega(x) - y | \cdot \|\phi_\omega(x)\|_2 \bigr] \leq 2Y_{\max} + 2R_T.
\#
By further plugging \eqref{eq::pf_nnsl_eq18} into \eqref{eq::pf_nnsl_eq21} into \eqref{eq::pf_nnsl_eq17}, we obtain for all $i \in [n]$ that
\#\label{eq::pf_nnsl_eq22}
\|\omega_i - W_{\text{\rm init}}\|_2 \leq \bigl((1 + 2\eta)\cdot R_T + 2\eta\cdot Y_{\max}\bigr).
\#
Meanwhile, similar to \eqref{eq::pf_nnsl_eq13}, we obtain for $\omega, v\in\cB_T$ that
\#\label{eq::pf_nnsl_eq23}
\|K_{\omega, \eta} (\omega-v)\|_2 \leq (2 + 4\eta)\cdot\|\omega-v\|_2\leq (4 + 8\eta)R_T.
\#
Finally, by Assumption \ref{asu::reg_cond_sl} with Lemma \ref{lem::lin_err}, we obtain for $U$ in \eqref{eq::pf_nnsl_eq16} that
\#\label{eq::pf_nnsl_eq24}
U &\leq 2\EE_{\text{\rm init}}\Bigl[\bigl\| \bigl(\phi_{\omega_i}(\cdot) - \phi_{0}(\cdot)\bigr)^\top \bigl(K_{\omega, \eta} (\omega-v) +W_{\text{\rm init}}\bigr)\bigr\|^2_\rho\Bigr] + 2\EE_{\text{\rm init}}\bigl[\| \phi_{\omega_i}(\cdot)^\top W_{\text{\rm init}} - \phi_{0}^\top W_{\text{\rm init}} \|^2_\rho\bigr]\notag\\
&= \cO\Bigl(\bigl((1 +  \eta)\cdot R_T + \eta\cdot Y_{\max}\bigr)^{3}\cdot m^{-1/2}\Bigr).
\#
Meanwhile, by \eqref{eq::pf_sqloss_eq8} in the proof of Corollary \ref{cor::opt_sl_sqloss} in \S\ref{pf::opt_sl_sqloss}, we obtain that
\#\label{eq::pf_nnsl_C_i}
\EE_{\text{\rm init}}\bigl[\|\delta R_i/\delta h_{\omega_i}\|^2_\rho\bigr] \leq 4 \EE_{\text{\rm init}}[R_i(h_{\omega_i})].
\#
Thus, by plugging \eqref{eq::pf_nnsl_eq24} and \eqref{eq::pf_nnsl_eq25} into \eqref{eq::pf_nnsl_eq16}, we have
\#\label{eq::pf_nnsl_eq25}
\EE_{\text{\rm init}}[P_i] = \cO(C_i \cdot R_1^{3/2}\cdot m^{-1/4}),
\#
where we define $C_i$ and $R_1$ as follows,
\#\label{eq::pf_nnsl_eq26}
C_i = 2\bigl\{\EE_{\text{\rm init}}[R_i(h_{\omega_i})]\bigr\}^{1/2}, \qquad R_1 =(1 +  \eta)\cdot R_T + \eta\cdot Y_{\max}.
\#

Finally, by plugging \eqref{eq::pf_nnsl_w}, \eqref{eq::pf_nnsl_eq15}, and \eqref{eq::pf_nnsl_eq25} into \eqref{eq::pf_nnsl_eq11}, we conclude that
\#\label{eq::111222}
\EE_{\text{\rm init}}\bigl[L(\omega) - L(\theta^*) \bigr]&\leq \epsilon + \EE_{\text{\rm init}}\bigl[2\overline R\cdot \inf_{\theta \in \cB_0} \|u(\cdot) - f(\cdot; \theta)\|_\rho\bigr] \\
&\qquad+ \cO(\overline C\cdot (1 + 2\eta)^{3/2}\cdot R_T^{3/2}\cdot m^{-1/4} + D_0\cdot R_1^{3/2}\cdot m^{-1/4}).\notag
\#
Here $\cB_0$ and $\overline C$ are the constants defined in \eqref{eq::pf_nnsl_B0} and \eqref{eq::def_C_bar}, respectively, $R_1 = (1 +  \eta)\cdot R_T + \eta\cdot Y_{\max}$, and
$
 D_0 = \frac{1}{n}\cdot \sum^n_{i = 1} C_i,
$
where $C_i$ is the constant defined in \eqref{eq::pf_nnsl_eq26}. Thus, by setting $G_T = (1  + \eta)\cdot R_T + \eta\cdot Y_{\max}$, we obtain from \eqref{eq::111222} that
\$
\EE_{\text{\rm init}}\bigl[L(\omega) - L(\theta^*) \bigr]\leq \epsilon + \EE_{\rm init} \Bigl[2\overline R \cdot \inf_{v\in\cB_{0}} \|\overline u(\cdot) - f(\cdot; v)\|_{\rho} \Bigr]+  \cO(G_T^{3/2}\cdot m^{-1/4}),
\$
which completes the proof of Corollary \ref{cor::nnsl}.
\end{proof}

\section{Proof of Auxiliary Result}
In this section, we present the proofs fo the auxiliary results.
\subsection{Proof of Proposition \ref{prop::meta_grad}}
\label{pf::meta_grad}
\begin{proof}
By Lemma \ref{lem::pg_thm}, which is the policy gradient theorem \citep{sutton2018reinforcement}, we have
\#\label{eq::pf_meta_grad_eq1}
\nabla_\theta J_i(\pi_{i, \theta}) = \EE_{(s, a)\sim \sigma_{\pi_{i, \theta}}}\bigl[ \nabla_\theta \log \pi_{i, \theta}(a\given s) \cdot A_i^{\pi_{i, \theta}}(s, a)\bigr],
\#
where recall that $J_i$ and $A_i^{\pi_{i, \theta}}$ are the expected total reward and the advantage function of the policy $\pi_{i, \theta}$ corresponding to the MDP $(\cS, \cA, P_i, r_i, \gamma_i, \zeta_i)$, respectively, and $\sigma_{\pi_{i, \theta}}$ is the state-action visitation measure induced by the policy $\pi_{i, \theta}$. By plugging the form of $\pi_{i, \theta}$ in \eqref{eq::PPO_optimal} into \eqref{eq::pf_meta_grad_eq1}, we obtain that
\#\label{eq::pf_meta_grad_eq2}
\nabla_\theta J_i(\pi_{i, \theta}) = \EE_{(s, a)\sim \sigma_{\pi_{i, \theta}}}\Bigl[\bigl(1/\tau\cdot \phi(s, a) +\eta\cdot \nabla_\theta Q^{\pi_{\theta}}_i(s, a)\bigr)\cdot A_i^{\pi_{i, \theta}}(s, a)\Bigr].
\#
Here $Q^{\pi_\theta}_i$ is the state-action value function of the main effect $\pi_{\theta}$ corresponding to the MDP $(\cS, \cA, P_i, r_i, \gamma_i, \zeta_i)$. Applying Lemma \ref{lem::pg_thm} again, we obtain that
\#\label{eq::pf_meta_grad_eq3}
\nabla_\theta Q^{\pi_{\theta}}_i(s, a) &=\nabla_\theta \Bigl((1 - \gamma_i)\cdot r_i(s, a) + \gamma_i\cdot\EE_{s'\sim P_i(\cdot\given s, a)}\bigl[V^{\pi_\theta}_i(s')\bigr]\Bigr) \notag\\
&=\gamma_i\cdot \EE_{(s', a')\sim \sigma_{i, \pi_\theta}^{(s, a)}}\bigl[\nabla_\theta \log  \pi_\theta(a'\given s')\cdot A^{\pi_\theta}_i(s', a')\bigr]\notag\\
&= \gamma_i\cdot\EE_{(s', a')\sim \sigma_{i, \pi_\theta}^{(s, a)}}\Bigl[\bigl(1/\tau \cdot \phi(s', a')\bigr)\cdot A^{\pi_\theta}_i(s', a')\Bigr].
\#
Here the last equality follows from the parameterization of $\pi_\theta$ in \eqref{eq::def_leader} and $\sigma_{i, \pi_\theta}^{(s, a)}$ is the state-action visitation measure of the main effect $\pi_\theta$, which is defined in Definition \ref{def::meta_visit}. By plugging \eqref{eq::pf_meta_grad_eq3} into \eqref{eq::pf_meta_grad_eq2}, we complete the proof of Proposition \ref{prop::meta_grad}.
\end{proof}

\subsection{Proof of Lemma \ref{lem::meta_grad_refine}}
\label{pf::meta_grad_refine}
\begin{proof}
Following from Proposition \ref{prop::meta_grad}, we have
\#\label{eq::pf_ref_eq1}
\nabla_\theta L(\theta) = \frac{1}{n}\cdot \sum^n_{i = 1} \EE_{(s', a', s, a)\sim \rho_{i, \pi_\theta} }\bigl[&\gamma_i\cdot\eta/\tau\cdot\phi(s', a')\cdot A^{\pi_\theta}_i(s', a')\cdot A_i^{\pi_{i, \theta}}(s, a) \\
&\quad+ 1/\tau\cdot\phi(s, a)\cdot A_i^{\pi_{i, \theta}}(s, a)\bigr],\notag
\#
where $\rho_{i, \pi_\theta}$ is the joint meta-visitation measure defined in \eqref{eq::def_meta_joint} of Definition \ref{def::meta_meta_visit}. Meanwhile, it holds that
\#\label{eq::pf_ref_eq2}
&\EE_{(s', a', s, a)\sim \rho_{i, \pi_\theta} }\bigl[\phi(s', a')\cdot A^{\pi_\theta}_i(s', a')\cdot A^{\pi_{i, \theta}}_i(s, a)\bigr]\notag\\
&\quad = \EE_{(s', a')\sim \varsigma_{i, \pi_\theta}}\bigl[\phi(s', a')\cdot G_{i, \pi_\theta}(s', a')\cdot A^{\pi_\theta}_i(s', a') \bigr],
\#
where $\varsigma_{i, \pi_\theta}$ is the meta-visitation measure defined in \eqref{eq::def_meta_visit} of Definition \ref{def::meta_meta_visit}, and $G_{i, \pi_\theta}$ is defined as follows,
\#\label{eq::pf_ref_def_A}
G_{i, \pi_\theta}(s', a') = \EE_{(s',a',s,a)\sim \rho_{i, \pi_\theta}}\bigl[A^{\pi_{i, \theta}}_i(s, a) \,\big|\, s', a'\bigr].
\#
Here $\rho_{i, \pi_\theta}$ is the joint meta-visitation measure defined in \eqref{eq::def_meta_joint} of Definition \ref{def::meta_meta_visit}.
By plugging \eqref{eq::pf_ref_eq2} into \eqref{eq::pf_ref_eq1}, we obtain that
\$
\nabla_\theta L(\theta) &= \frac{1}{n}\cdot \sum^n_{i = 1} \EE_{(s, a)\sim \sigma_{\pi_{i, \theta}}}\bigl[ 1/\tau \cdot \phi(s, a)\cdot A^{\pi_{i, \theta}}_i(s, a) \bigr]\notag\\
&\notag\quad\quad\qquad\quad + \EE_{(s', a')\sim \varsigma_{i, \pi_\theta}}\bigl[\gamma_i\cdot\eta/\tau \cdot\phi(s', a')\cdot G_{i, \pi_\theta}(s', a')\cdot A^{\pi_{\theta}}_i(s', a') \bigr]\notag\\
&= \frac{1}{n}\cdot \sum^n_{i = 1} \EE_{(s', a')\sim \varsigma_{i, \pi_\theta}}\Bigl[\phi(s', a')\cdot \bigl(1/\tau \cdot A^{\pi_{i, \theta}}_i(s', a') \cdot \bigl(\ud \sigma_{\pi_{i, \theta}}/\ud \varsigma_{i, \pi_\theta}(s', a')\bigr) \\
&\qquad\qquad\qquad\qquad\qquad + \gamma_i\cdot\eta/\tau\cdot G_{i, \pi_\theta}(s', a')\cdot A^{\pi_{\theta}}_i(s', a')\bigr)\Bigr],\notag
\$
where $A_{i, \theta}$ is defined in \eqref{eq::pf_ref_def_A} and $\ud \sigma_{\pi_{i, \theta}}/\ud \varsigma_{i, \pi_\theta}$ is the Radon-Nikodym derivative. Thus, we complete the proof of Lemma \ref{lem::meta_grad_refine}.
\end{proof}

\subsection{Proof of Proposition \ref{prop::sqloss_diff}}
\label{pf::sqloss_diff}
\begin{proof}
It suffices to prove for all $h \in \cH$ and $\delta R_i/\delta_h$ defined in \eqref{eq::sqloss_diff} that the operator
\#\label{eq::pf_def_A}
A_{h}(\cdot) = \langle\cdot, \delta R_i/\delta_h\rangle_{\cH}
\#
is the Fr\'echet derivative of $R_i$ at $h\in\cH$ defined in Definition \ref{def::fre_diff}. For all $h_1 \in\cH$, we have
\#\label{eq::pf_sqloss_diff_eq1}
&R_i(h_1) - R_i(h) - A_h(h_1 - h) \notag\\
&\quad= \int_{\cX\times\cY} \bigl(h_1(x) - y\bigr)^2 - \bigl(h(x) - y\bigr)^2 \ud \cD_i(x, y) - \int_\cX \bigl(h_1(x) - h(x)\bigr)\cdot (\delta R_i/\delta_h)(x) \ud \rho(x),
\#
where the equality follows from the definition of $A_h$ in \eqref{eq::pf_def_A}. Meanwhile, for $\delta R_i/\delta h$ defined in \eqref{eq::sqloss_diff} and $\overline \cD_i = \rho$, we have
\#\label{eq::pf_sqloss_diff_eq2}
\int_\cX \bigl(h_1(x) - h(x)\bigr)\cdot (\delta R_i/\delta_h)(x) \ud \rho(x) = \int_\cX 2\bigl(h_1(x) - h(x)\bigr)\cdot\bigl (h(x) - y\bigr) \ud \cD_i(x, y).
\#
By plugging \eqref{eq::pf_sqloss_diff_eq2} into \eqref{eq::pf_sqloss_diff_eq1}, we obtain that
\#\label{eq::pf_sqloss_diff_eq3}
&R_i(h_1) - R_i(h) - A_h(h_1 - h) \notag\\
&\quad =  \int_{\cX\times\cY} \bigl(h_1(x) - y\bigr)^2 - \bigl(h(x) - y\bigr)^2 -  2\bigl(h_1(x) - h(x)\bigr)\cdot\bigl (h(x) - y\bigr)\ud \cD_i(x, y)\notag\\
&\quad = \int_{\cX\times\cY}\bigl(h_1(x) - h(x)\bigr)\cdot\bigl(h_1(x) + h(x) - 2y\bigr) -  2\bigl(h_1(x) - h(x)\bigr)\cdot\bigl (h(x) - y\bigr) \ud \cD_i(x, y)\notag\\
&\quad = \int_\cX \bigl(h_1(x) - h(x)\bigr)^2 \ud \rho(x) = \|h_1 - h\|^2_\cH.
\#
Hence, by \eqref{eq::pf_sqloss_diff_eq3}, we have
\$
\lim_{\|h_1 - h\|_\cH \to 0} \frac{|R_i(h_1) - R_i(h) - A_h(h_1 - h)|}{\|h_1 - h\|_\cH} = \lim_{\|h_1 - h\|_\cH \to 0} \|h_1 - h\|_\cH = 0.
\$
Thus, following from the definition of Fr\'echet derivative in Definition \ref{def::fre_diff}, we conclude that $A_h$ defined in \eqref{eq::pf_def_A} is the Fr\'echet derivative of $R_i$ at $h\in\cH$, which completes the proof of Proposition \ref{prop::sqloss_diff}.
\end{proof}

\subsection{Proof of Lemma \ref{lem::lin_err}}
\label{pf::lin_err}
\begin{proof}
In what follows, we write $\phi_{W_{\text{\rm init}}} = \phi_0$ for notational simplicity, where $\phi_{W_{\text{\rm init}}}$ is the feature mapping defined in \eqref{eq::def_nn_phi} with $W = W_{\text{\rm init}}$. Note that
\#\label{eq::linerr_pf_1}
&\|\phi_{\omega_0}(\cdot,\cdot)^\top \omega_2 - \phi_{\omega_1}(\cdot, \cdot)^\top\omega_2\|^2_{\varrho_{\pi_\theta}}\notag\\&\qquad\leq 2\|\phi_{\omega_0}(\cdot,\cdot)^\top \omega_2 - \phi_{0}(\cdot, \cdot)^\top\omega_2\|^2_{\varrho_{\pi_\theta}} + 2\|\phi_{\omega_1}(\cdot,\cdot)^\top \omega_2 - \phi_{0}(\cdot, \cdot)^\top\omega_2\|^2_{\varrho_{\pi_\theta}},
\#
where the inequality follows from the fact that $\|f(\cdot) + g(\cdot)\|^2_{\varrho_{\pi_\theta}} \leq 2\|f(\cdot)\|^2_{\varrho_{\pi_\theta}} + 2\|g(\cdot)\|^2_{\varrho_{\pi_\theta}}$. We now upper bound the right-hand side of \eqref{eq::linerr_pf_1} under the expectation with respect to the random initialization. In the sequel, we write $\phi_\omega(\cdot) = ([\phi_\omega(\cdot)]_{\text{\rm U}}^\top, [\phi_\omega(\cdot)]_{\text{\rm L}}^\top)^\top$ and $\phi_0(\cdot) = ([\phi_0(\cdot)]_{\text{\rm U}}^\top, [\phi_0(\cdot)]_{\text{\rm L}}^\top)^\top$, respectively, where $[\phi_0(\cdot)]_{\text{\rm U}} = ([\phi_0(\cdot)]^\top_1, \ldots, [\phi_0(\cdot)]^\top_{m/2})^\top$ and $[\phi_0(\cdot)]_{\text{\rm L}} = ([\phi_0(\cdot)]^\top_{m/2+1}, \ldots, [\phi_0(\cdot)]^\top_{m})^\top$. Similarly, we write $\omega = ([\omega]_{\text{\rm U}}^\top, [\omega]_{\text{\rm L}}^\top)^\top$ and $W_{\text{\rm init}} = ([W_{\text{\rm init}}]_{\text{\rm U}}^\top, [W_{\text{\rm init}}]_{\text{\rm L}}^\top)^\top$, respectively. Note that for $\|\omega - W_{\text{\rm init}}\|_2 \leq R$, we have
\#\label{eq::linerr_pf_2}
\|\omega - W_{\text{\rm init}}\|^2_2 = \|[\omega]_{\text{\rm U}} - [W_{\text{\rm init}}]_{\text{\rm U}}\|^2_2 + \|[\omega]_{\text{\rm L}} - [W_{\text{\rm init}}]_{\text{\rm L}}\|^2_2 \leq R^2.
\#
Hence, we obtain from \eqref{eq::linerr_pf_2} that $\|[\omega]_{\text{\rm U}} - [W_{\text{\rm init}}]_{\text{\rm U}}\|_2 \leq R$ and $\|[\omega]_{\text{\rm L}} - [W_{\text{\rm init}}]_{\text{\rm L}}\|_2 \leq R$. Meanwhile, note that $[W_{\text{\rm init}}]_{\text{\rm U}} = ([W_{\text{\rm init}}]_1^\top, \ldots, [W_{\text{\rm init}}]_{m/2}^\top)^\top$, where $[W_{\text{\rm init}}]_r\sim N(0, I_d/d)$ are mutually independent for all $r\in[m/2]$. Thus, by Lemma \ref{lem::lin_err_origin}, under Assumption \ref{asu::reg_cond_rl}, we obtain for $\omega_0, \omega_2 \in \cB = \{\theta\in\RR^{md}:\|\theta - W_{\text{\rm init}}\|_2 \leq R\}$ that
\#\label{eq::linerr_pf_3}
\EE_{\text{\rm init}}\bigl[\| [\phi_{\omega_0}(\cdot)]_{\text{\rm U}}^\top [\omega_2]_{\text{\rm U}} -  [\phi_{0}(\cdot)]_{\text{\rm U}}^\top [\omega_2]_{\text{\rm U}} \|_{\varrho_{\pi_\theta}}^2\bigr] = \cO(R^3\cdot m^{-1/2}).
\#
Similarly, we have
\#\label{eq::linerr_pf_4}
\EE_{\text{\rm init}}\bigl[\| [\phi_{\omega_0}(\cdot)]_{\text{\rm L}}^\top [\omega_2]_{\text{\rm L}} -  [\phi_{0}(\cdot)]_{\text{\rm L}}^\top [\omega_2]_{\text{\rm L}} \|_{\varrho_{\pi_\theta}}^2\bigr] = \cO(R^3\cdot m^{-1/2}).
\#
Following from \eqref{eq::linerr_pf_3} and \eqref{eq::linerr_pf_4}, we have
\#\label{eq::linerr_pf_5}
&\EE_{\text{\rm init}}\bigl[\|\phi_{\omega_0}(\cdot,\cdot)^\top \omega_2 - \phi_{0}(\cdot, \cdot)^\top\omega_2\|^2_{\varrho_{\pi_\theta}}\bigr]\notag\\
&\qquad= \EE_{\text{\rm init}}\bigl[\|[\phi_{\omega_0}(\cdot)]_{\text{\rm U}}^\top [\omega_2]_{\text{\rm U}} -  [\phi_{\omega_0}(\cdot)]_{\text{\rm U}}^\top [\omega_2]_{\text{\rm U}} + [\phi_{\omega_0}(\cdot)]_{\text{\rm L}}^\top [\omega_2]_{\text{\rm L}} -  [\phi_{\omega_0}(\cdot)]_{\text{\rm L}}^\top [\omega_2]_{\text{\rm L}} \|^2_{\varrho_{\pi_\theta}}\bigr]\notag\\
&\qquad \leq 2\EE_{\text{\rm init}}\bigl[\| [\phi_{\omega_0}(\cdot)]_{\text{\rm U}}^\top [\omega_2]_{\text{\rm U}} -  [\phi_{0}(\cdot)]_{\text{\rm U}}^\top [\omega_2]_{\text{\rm U}} \|_{\varrho_{\pi_\theta}}^2\bigr] + 2 \EE_{\text{\rm init}}\bigl[\| [\phi_{\omega_0}(\cdot)]_{\text{\rm L}}^\top [\omega_2]_{\text{\rm L}} -  [\phi_{0}(\cdot)]_{\text{\rm L}}^\top [\omega_2]_{\text{\rm L}} \|_{\varrho_{\pi_\theta}}^2\bigr]\notag\\
&\qquad =\cO(R^{3}\cdot m^{-1/2}),
\#
where the second inequality follows from the fact that $\|f(\cdot) + g(\cdot)\|^2_{\varrho_{\pi_\theta}} \leq 2\|f(\cdot)\|^2_{\varrho_{\pi_\theta}} + 2\|g(\cdot)\|^2_{\varrho_{\pi_\theta}}$. Similarly, we have
\#\label{eq::linerr_pf_6}
\EE_{\text{\rm init}}\bigl[\|\phi_{\omega_1}(\cdot,\cdot)^\top \omega_2 - \phi_{0}(\cdot, \cdot)^\top\omega_2\|^2_{\varrho_{\pi_\theta}}\bigr] = \cO(R^{3}\cdot m^{-1/2}).
\#
Finally, by plugging \eqref{eq::linerr_pf_5} and \eqref{eq::linerr_pf_6} into \eqref{eq::linerr_pf_1}, we have
\$
&\EE_{\text{\rm init}}\bigl[\|\phi_{\omega_0}(\cdot,\cdot)^\top \omega_2 - \phi_{\omega_1}(\cdot, \cdot)^\top\omega_2\|^2_{\varrho_{\pi_\theta}}\bigr]\notag\\
&\qquad\leq 2\EE_{\text{\rm init}}\bigl[\|\phi_{\omega_0}(\cdot,\cdot)^\top \omega_2 - \phi_{0}(\cdot, \cdot)^\top\omega_2\|^2_{\varrho_{\pi_\theta}}\bigr] + 2\EE_{\text{\rm init}}\bigl[\|\phi_{\omega_1}(\cdot,\cdot)^\top \omega_2 - \phi_{0}(\cdot, \cdot)^\top\omega_2\|^2_{\varrho_{\pi_\theta}}\bigr]\notag\\
&\qquad = \cO(R^{3}\cdot m^{-1/2}),
\$
which concludes the proof of Lemma \ref{lem::lin_err}.
\end{proof}

\section{Auxiliary Lemma}
In this section, we present the auxiliary lemmas.

\begin{lemma}[Linearization Error \citep{cai2019neural}]
\label{lem::lin_err_origin}
Let $\|x\| \leq 1$ for all $x\in\cX$ and $ [W_{\text{\rm init}}]_r \sim N(0, I_d/d)$ be mutually independent for all $r\in[m]$. For parameters $\omega, \omega'\in \cB_{\text{\rm init}} = \{\theta \in \RR^{md}: \|\theta - W_{\text{\rm init}}\|_2\leq R\}$ and the distribution $\rho$ over $\cX$ such that Assumption \ref{asu::reg_cond_sl} holds, we have
\$
\EE_{\text{\rm init}}\bigl[\|\phi_{\omega}(\cdot)^\top\omega' - \phi_{W_{\text{\rm init}}}(\cdot)^\top\omega'\|_\rho^2\bigr] = \cO(R^3\cdot m^{-1/2}).
\$
\end{lemma}
\begin{proof}
See \cite{cai2019neural} for a detailed proof.
\end{proof}

\begin{lemma}[Policy Gradient \citep{sutton2018reinforcement}]
\label{lem::pg_thm}
Let $\pi_\theta$ be the parameterized policy with the parameter $\theta$. It holds that
\$
\nabla_\theta J(\pi_\theta) &= \EE_{s\sim \nu_{\pi_\theta}}\bigl[ \langle \pi_\theta(\cdot \given s),   Q^{\pi_\theta}(s, \cdot)\rangle \bigr]\\
&= \EE_{(s, a)\sim \sigma_{\pi_\theta}}\bigl[\nabla_\theta \log \pi_\theta(a\given s)\cdot Q^{\pi_\theta}(s, a)\bigr],
\$
where $\nu_{\pi_\theta}$ is the state visitation measure defined in \eqref{eq::def_visit} with $\pi = \pi_\theta$, and  $\sigma_{\pi_{\theta}}(\cdot, \cdot) = \pi_\theta(\cdot\given\cdot)\cdot \nu_{\pi_\theta}(\cdot)$ is the corresponding state-action visitation measure induce by $\pi_\theta$.
\end{lemma}
\begin{proof}
See \cite{sutton2018reinforcement} for a detailed proof.
\end{proof}

\begin{lemma}[Performance Difference \citep{kakade2002approximately}]
\label{lem::performance_diff}
It holds for all policies $\pi$ and $\tilde \pi$ that
\$
J(\tilde\pi) - J(\pi) = (1 - \gamma)^{-1}\cdot \EE_{(s, a)\sim \sigma_{\tilde\pi}}\bigl[ A^{\pi}(s, a)\bigr],
\$
where $\sigma_{\tilde\pi}$ is the state-action visitation measure induced by $\tilde\pi$.
\end{lemma}
\begin{proof}
See \cite{kakade2002approximately} for a detailed proof.
\end{proof}
\end{document}